\documentclass[11pt,oneside]{article}

\usepackage{config_mjh_paper_arxiv}
\usepackage{config_mjh_symbols_arxiv}

\usepackage[numbers]{natbib}
\usepackage{amsmath}
\usepackage{amsthm}
\theoremstyle{definition} \newtheorem{defn}{Definition}
\theoremstyle{plain} 
\theoremstyle{plain} \newtheorem{thm}[defn]{Theorem}
\theoremstyle{plain} \newtheorem{lem}[defn]{Lemma}
\theoremstyle{plain} 
\theoremstyle{remark} \newtheorem{rmk}[defn]{Remark}
\theoremstyle{remark} 

\usepackage{enumitem} %
\makeatletter
\def\namedlabel#1#2{\begingroup
    #2%
    \def\@currentlabel{#2}%
    \phantomsection\label{#1}\endgroup
}
\makeatother

\usepackage[pdfauthor={Matthew J Holland},%
colorlinks=true,%
linkcolor=blue,%
citecolor=blue]{hyperref}
\hypersetup{pdftitle={Holland (2020): Making learning more transparent using conformalized performance prediction}} %

\begin{document}

\title{\textbf{Making learning more transparent using\\conformalized performance prediction}}
\author{
  Matthew J.~Holland\thanks{Please direct correspondence to \texttt{matthew-h@ar.sanken.osaka-u.ac.jp}.}\\
  Osaka University
}
\date{} 

\maketitle

\begin{abstract}
In this work, we study some novel applications of conformal inference techniques to the problem of providing machine learning procedures with more transparent, accurate, and practical performance guarantees. We provide a natural extension of the traditional conformal prediction framework, done in such a way that we can make valid and well-calibrated predictive statements about the future performance of arbitrary learning algorithms, when passed an as-yet unseen training set. In addition, we include some nascent empirical examples to illustrate potential applications.
\end{abstract}

\tableofcontents

\newpage

\section{Introduction}\label{sec:intro}

As machine learning systems become increasingly entangled with human decision-making processes, it is of social, ethical, and economic importance to be able to understand and explain the \emph{limits} of what can be said about such technologies. We need to be able to formulate meaningful questions about the behavior or performance of learning systems that are accessible to a diverse audience, and we must have the tools required to provide unambiguous answers to these questions. From the viewpoint of responsible and sustainable system design, the tools used to \emph{answer questions about the learning system} are arguably even more important than the algorithms underlying the learning system itself.

In this context, there is an important line of research related to \term{conformal prediction}, a general-purpose methodology for attaching high-probability guarantees to data-driven predictors \citep{vovk2005ALRW,shafer2008a,lei2015a,lei2018a}. For example, if the user specifies a confidence level of say $90\%$ in advance, questions such as
\begin{center}
\textit{``I've trained a classifier. Given a new instance, what labels seem most likely?''}
\end{center}
can be given clear and meaningful answers by using conformal prediction to construct predictive sets that can be computed in practice, and which include the correct label(s) $90\%$ of the time. The key merit of this methodology is its generality: it can be applied to any classifier, with no special assumptions made on the underlying data distribution. Of course, this comes at a cost: one must sacrifice some training data, setting it aside for the calibration of prediction sets. Important early applications of conformal prediction were concerned with regression tasks \citep{lei2015a,lei2018a,tibshirani2020a,romano2020a}. More recently, the methodology has been adapted to handle classification tasks \citep{cauchois2020a,romano2020b}, and also more general decision-making tasks using predictive probability distributions instead of predictive sets \citep{vovk2018a,vovk2019a}. In all this existing literature, the trained predictor (the ``base predictor'') plays an ancillary role. That is, the only kinds of questions that can be answered with these techniques are those relating to the \emph{prediction of new labels, given a base predictor}.

In this work, we explore the possibilities for using conformal prediction methods to answer a wider class of questions about the performance of machine learning algorithms. Instead of predicting new labels given a base predictor, our focus lies with reliable \emph{prediction of performance on new data, given a base algorithm}.

\paragraph{Shifting the focus to performance metrics}

As a bit of motivation, first let us recall that the standard formulation of learning tasks (in the context of statistical learning theory) is as risk minimization problems \citep{haussler1992a,vapnik1998SLT,vapnik1999NSLT}. That is, assuming random data $Z \sim \ddist$ and a \term{loss} function $\loss(h;z)$ used to evaluate arbitrary candidate $h$, the goal is to minimize the expected loss, or \term{risk}, defined as a function of $h$ by $\risk_{\ddist}(h) \defeq \exx_{\ddist}\loss(h;Z) = \int_{\ZZ} \loss(h;z) \, \ddist(\dif z)$.
Here \term{performance} is measured using the risk induced by $\loss$, and thus performance guarantees for a given learning algorithm $\algo$ are inevitably statements about the nature of the risk incurred by outputs of $\algo$. Owing to the popularity of the PAC-learning framework \citep{valiant1984a,kearns1994CLTIntro}, arguably the most common form for performance guarantees to take is that of high-probability, finite-sample bounds on the risk incurred by $\algo_{n} \defeq \algo(\Z_{n})$, where $\Z_{n} = (Z_{1},\ldots,Z_{n})$ is a sample of independent copies of $Z \sim \ddist$. That is, one derives bounds $\varepsilon(n,\alpha,\ddist)$ depending on the sample size, desired confidence level $\alpha$, and underlying distribution $\ddist$ such that
\begin{align}\label{eqn:guarantee_pac}
\prr\left\{ \risk_{\ddist}\left(\algo_{n}\right) \leq \varepsilon(n,\alpha,\ddist) \right\} \geq 1-\alpha.
\end{align}
On one hand, this is appealing since it gives us a high-probability guarantee over the random draw of the sample $\Z_{n}$ passed to algorithm $\algo$. The risk is one quantity that formally captures the intuitive notion of \emph{off-sample performance}, and since it is defined by taking expectation, it is not too difficult to obtain reasonably tight upper bounds for many concrete classes of algorithm $\algo$. Such bounds can then be used as formal evidence for (or against) using certain algorithms for certain classes of learning problems, characterized by properties of $\ddist$ and the nature of $\loss(h;z)$.

On the other hand, while guarantees of the form given by (\ref{eqn:guarantee_pac}) are informative for broad classes of algorithms and learning tasks, their utility is extremely limited for any \emph{particular} machine learning application. The most obvious reason is the fact that the risk is an ideal quantity; since $\risk_{\ddist}$ depends on the unknown distribution $\ddist$, the risk itself can never be computed, and thus the high-probability ``good performance'' event can never actually be checked. Even with a ``test set'' allocated for approximating $\risk_{\ddist}(\algo_{n})$, in many cases the bound $\varepsilon(n,\alpha,\ddist)$ also depends on unknown quantities. Even with finite models, constructing tight data-based bounds is a highly non-trivial problem \citep{langford2004a}, and the issues only grow more severe with bounds for more common machine learning models, which tend to be very \emph{pessimistic}, and do not accurately reflect \emph{typical} performance \citep{nagarajan2020a}. This is an inevitable tradeoff that comes with formal guarantees that hold under very weak assumptions on the underlying data-generating process.

With these issues in mind, we consider a different approach, in which we utilize the key statistical principles underlying conformal prediction to obtain \emph{computable} quantities that enable reliable predictions about the performance metric (here $\loss$ depending on $\algo$), for which we have rigorous guarantees of validity and calibration. Since our focus is on performance metrics used in learning and/or evaluation, we call this general approach \term{conformal performance prediction} (CPP). In this initial study, we consider two distinct classes of CPP sets:
\begin{enumerate}
\item \textbf{Candidate CPP:} having executed a learning algorithm $\algo$ and obtained a trained candidate $\algo_{\TR}$, construct a prediction set which reliably covers the loss incurred by this candidate at a new data point $Z$.
\begin{align}\label{eqn:cpp_candidate}
\prr\left\{ \loss(\algo_{\TR};Z) \in \widehat{\EE}_{\alpha}(\algo_{\TR}) \right\} \geq 1-\alpha.
\end{align}

\item \textbf{Algorithm CPP:} construct a prediction set which reliably covers the loss incurred at a new data point $Z$, after the algorithm is run on a new training set $\Z$.
\begin{align}\label{eqn:cpp_algo}
\prr\left\{ \loss(\algo(\Z);Z) \in \widehat{\EE}_{\alpha}(\algo) \right\} \geq 1-\alpha.
\end{align}
\end{enumerate}
A more detailed exposition and analysis will be given in section \ref{sec:cpp}, but for the moment, let us take a high-level look at the key traits of each approach. Clearly, both candidate CPP and algorithm CPP make predictions about the nature of as-yet unobserved losses based on already-observed data. The key difference comes from the source of randomness in the losses. In candidate CPP, we split the data into training and calibration subsets $\Z_{\TR}$ and $\Z_{\CP}$, run $\algo$ once on $\Z_{\TR}$, and then using $\Z_{\CP}$ construct a CPP set $\widehat{\EE}_{\alpha}$ to make predictions about how the \emph{particular} $\algo_{\TR} \defeq \algo(\Z_{\TR})$ we have learned will perform off-sample, i.e., when the new $Z$ is observed. In contrast, algorithm CPP also captures the uncertainty involved in the learning process itself, making predictions about off-sample performance when $\algo$ is run on a new sample $\Z$, and this can either be modulated by $Z$ or done independently of $Z$. This freshly-drawn sample $\Z$ appears only in algorithm CPP, it plays no role in the candidate-centric case. Intuitively, if we have guarantees of the form (\ref{eqn:cpp_candidate}), questions such as
\begin{center}
\textit{``Training is complete. Given a new example, what kind of performance can I expect?''}
\end{center}
can be given a reasonably clear answer. In contrast, questions such as
\begin{center}
\textit{``I've designed a new algorithm. Once trained, what kind of performance can I expect?''}
\end{center}
can only be answered by guarantees of the form (\ref{eqn:cpp_algo}). We give two illustrative examples in Figures \ref{fig:ex_optdigits}--\ref{fig:ex_quadnoise}.

\paragraph{Paper outline}

The remainder of the paper is structured as follows. In section \ref{sec:bg_conformal}, we provide a concise review of the key ideas and technical facts underlying conformal prediction, which will be utilized in the subsequent analysis. In section \ref{sec:cpp}, we formulate the CPP framework, providing validity guarantees, implementable algorithms, and a discussion of natural extensions to the base framework. Finally, in section \ref{sec:empirical}, we provide some nascent empirical examples meant to illustrate the way in which the proposed framework can be applied to a wide variety of learning tasks, and used to provide answers to questions about the learning procedures that could not be readily obtained using existing methods.

\begin{figure}[t]
\centering
\includegraphics[width=0.75\textwidth]{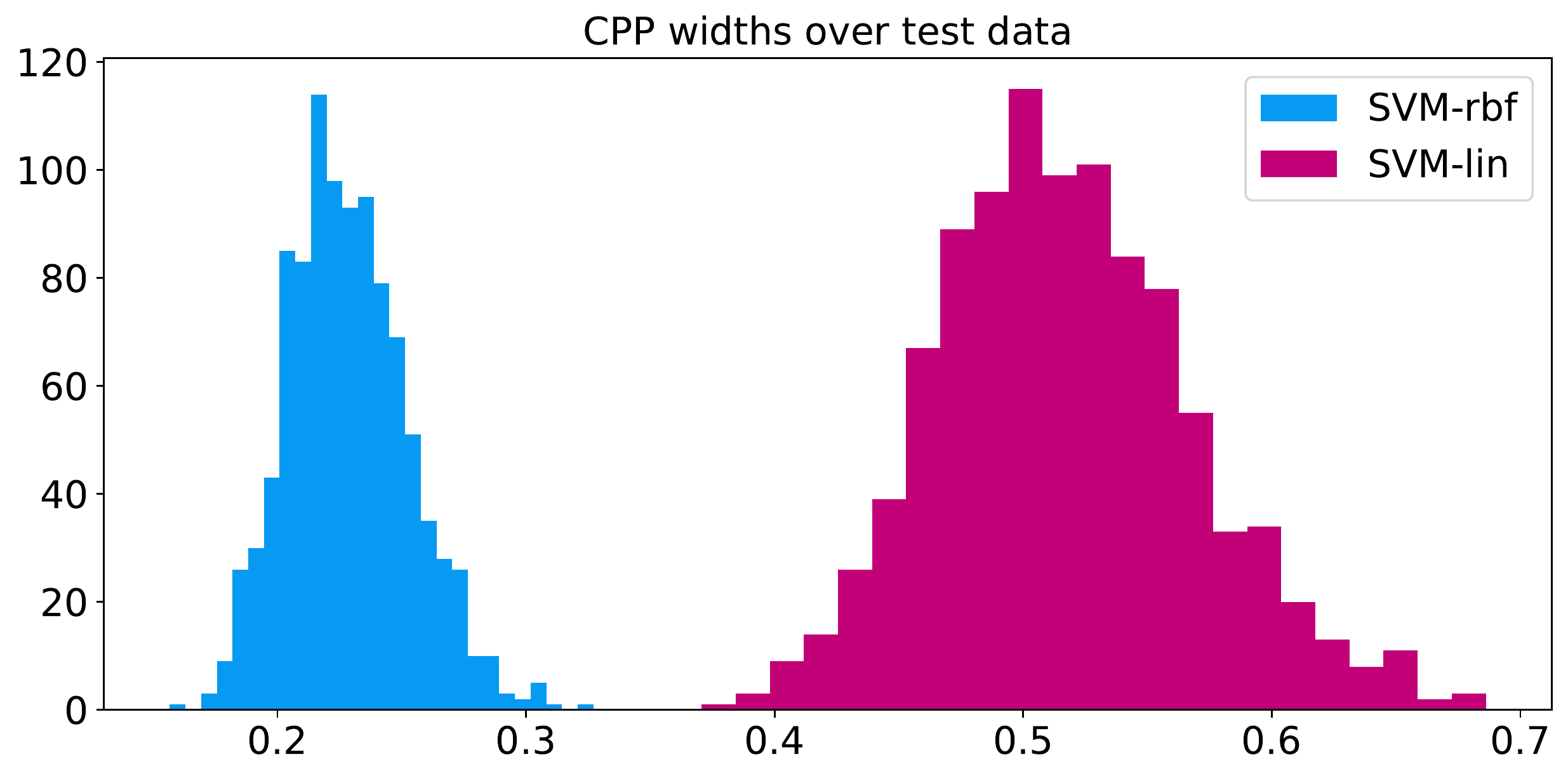}
\bigskip
\includegraphics[width=0.75\textwidth]{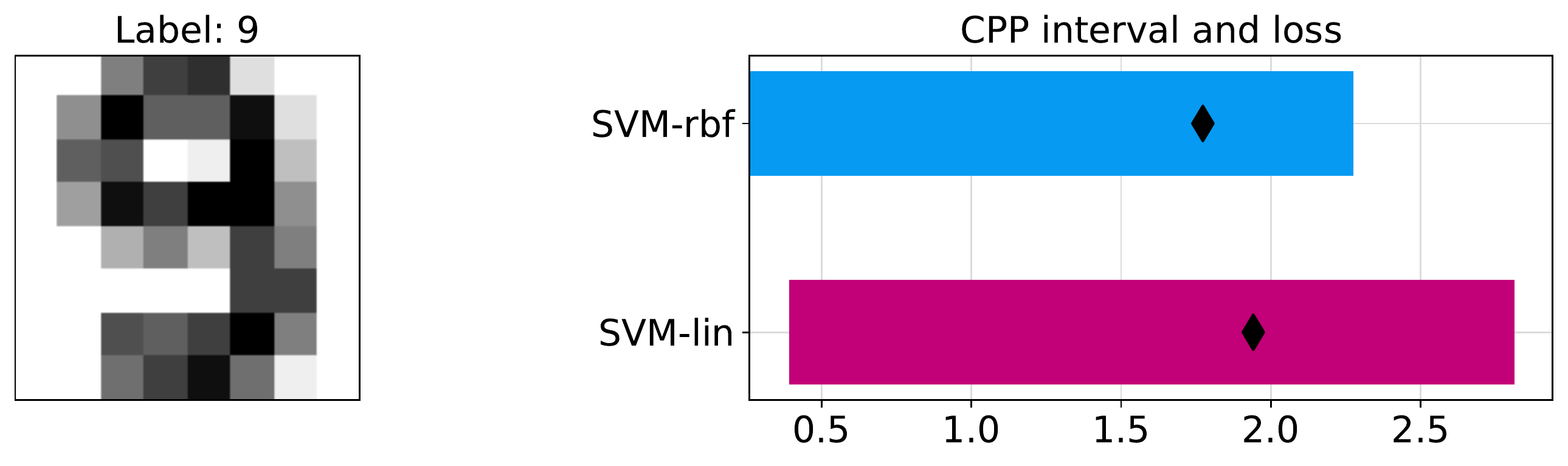}
\caption{An example of CPP applied to compare linear (magenta) vs.~non-linear (azure) SVM, run on the digits data, with performance measured using the logistic loss, at $90\%$ confidence. The top figure compares CPP interval widths over multiple randomized trials. The bottom figure (right-hand side) shows the exact CPP predicted for each method given over the random draw of a new training sample, evaluated at a particular new image (left-hand side). Black diamonds denote the actual losses incurred. See section \ref{sec:empirical_optdigits} for details.}
\label{fig:ex_optdigits}
\end{figure}

\begin{figure}[t]
\centering
\includegraphics[width=0.75\textwidth]{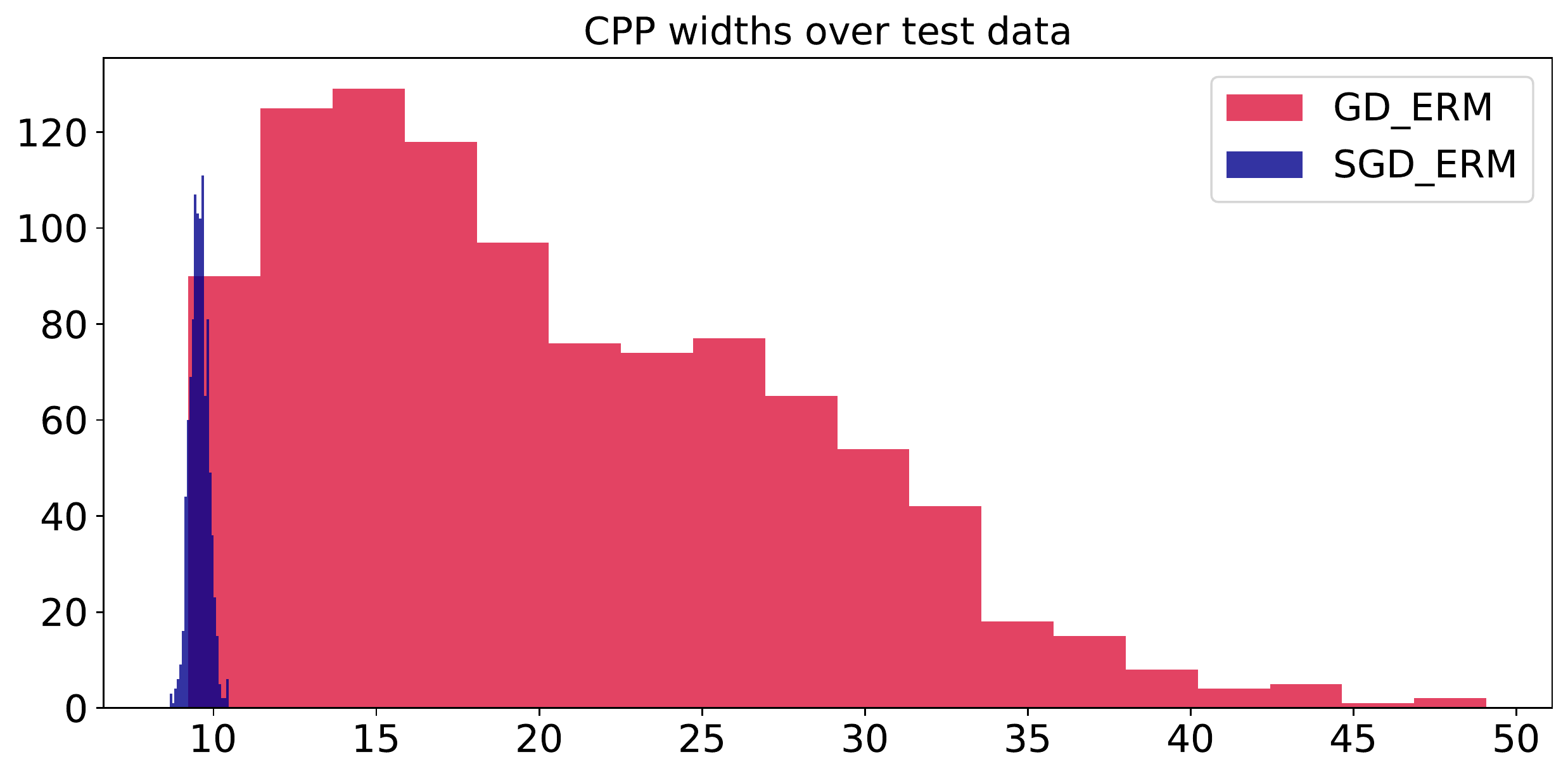}
\bigskip
\includegraphics[width=0.75\textwidth]{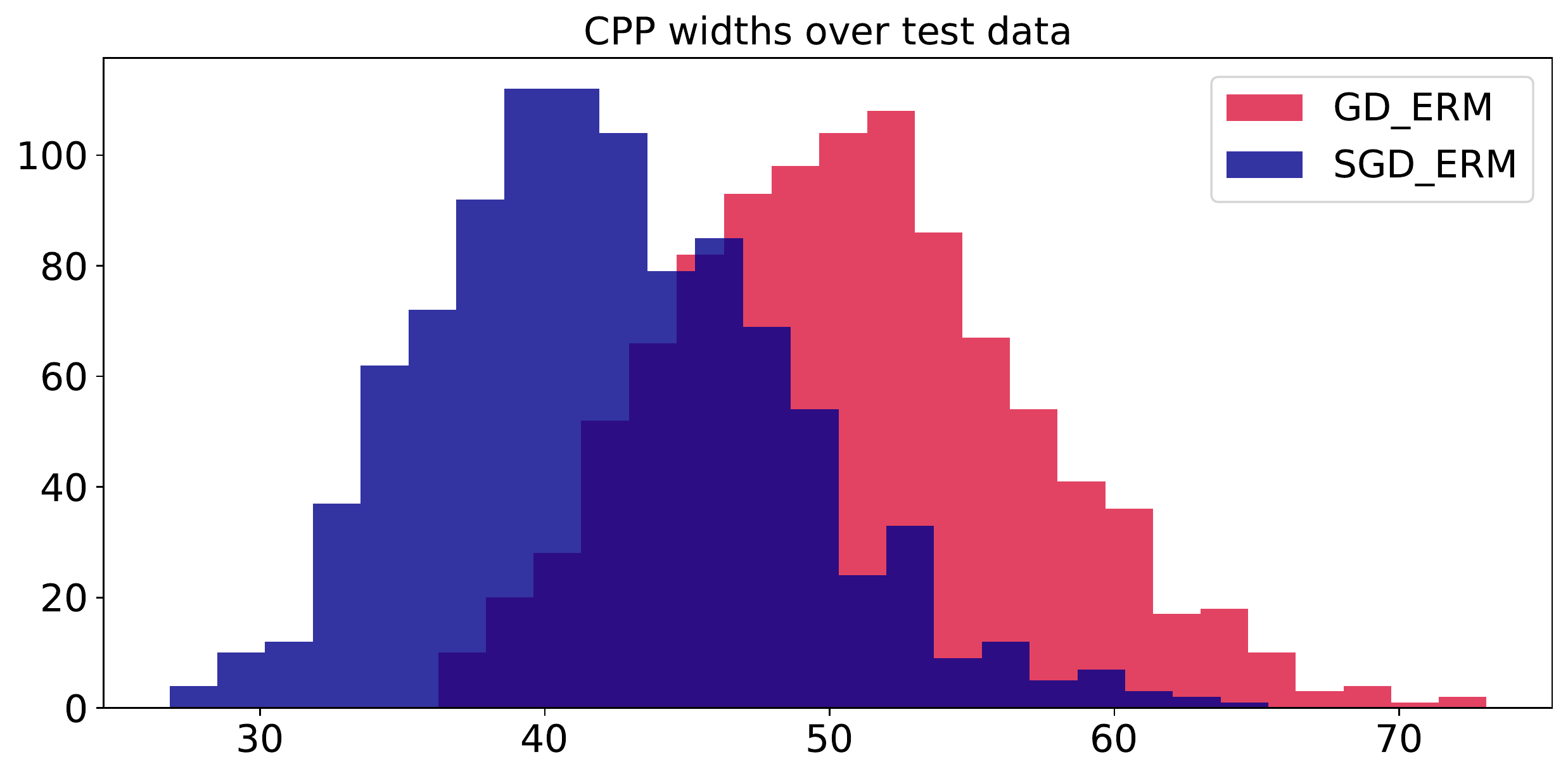}
\caption{An example of CPP applied to compare batch gradient descent (red) with stochastic gradient descent (blue), for a simple stochastic convex optimization problem, with performance measured using the squared error, at $90\%$ confidence. The top figure shows candidate CPP interval widths (section \ref{sec:cpp_candidate}), whereas the bottom figure shows $Z$-free algorithm CPP interval widths (section \ref{sec:cpp_algo_zfree}). Further details are given in section \ref{sec:empirical_quadnoise}.}
\label{fig:ex_quadnoise}
\end{figure}

\section{Background on conformal prediction}\label{sec:bg_conformal}

\paragraph{General notation}

Throughout this paper, we shall write random objects (both real-valued random variables and random vectors) by upper-case letters (e.g., $U$), non-random realizations as lower-case letters (e.g., $u$), and $n$-tuples (or \term{lists}, or \term{collections}) of random objects as bold-faced upper-case letters (e.g., $\U_{n} \defeq (U_{1},\ldots,U_{n})$). When we use the union operator $\cup$ for such lists, we simply concatenate the lists. For indexing purposes, write $[k] \defeq \{1,\ldots,k\}$ to denote the set of all positive integers no greater than $k$. When we want to denote a particular set of sub-scripted elements indexed by some $\II \subset [n]$ and the order of elements does not matter, we will use standard set notation such as $\{U_{i}: i \in \II\}$, but treat this as a tuple of length $|\II|$. Singleton lists and sets will be treated identically, with common notation $\{u\}$. We use $\prr$ and $\exx$ respectively as generic notation for probability and expectation; the underlying distribution of interest will be specified when not clear from the context. We write $\HH$ for a generic hypothesis class, namely a set of elements from which learning algorithms choose a particular final candidate in a data-driven manner. For any $x > 0$, the ceiling function $\lceil x \rceil$ returns the smallest integer greater than or equal to $x$. Similarly, the floor function $\lfloor x \rfloor$ returns the largest integer no greater than $x$.

\paragraph{Quantile background}

For an arbitrary random variable $U$ with distribution function $F(u) \defeq \prr\left\{ U \leq u \right\}$, we define the $\alpha$-level \term{quantile} of $U$ by
\begin{align*}
\qnt_{\alpha}[U] \defeq \inf \left\{ u \in \RR: F(u) \geq \alpha \right\}.
\end{align*}
Given $n$ observations $\U_{n} = (U_{1},\ldots,U_{n})$, we denote the $\alpha$-level empirical quantile by
\begin{align*}
\qnt_{\alpha}\left[ \U_{n} \right] \defeq \inf \left\{ u \in \RR: \widehat{F}_{n}(u) \geq \alpha \right\},
\end{align*}
where $\widehat{F}_{n}(u) \defeq (1/n)\sum_{i=1}^{n} I_{\{ U_{i} \leq u\}}$ is the empirical distribution function induced by the sample. Plugging in $u = \qnt_{\alpha}[\U_{n}]$, by the definition of the quantile function, we have
\begin{align}\label{eqn:quantile_ecdf_lowerbd}
\widehat{F}_{n}\left( \qnt_{\alpha}[\U_{n}] \right) \geq \alpha.
\end{align}
Taking expectation with respect to the sample we have
\begin{align*}
\exx \widehat{F}_{n}(u) = \frac{1}{n} \sum_{i=1}^{n} \prr\left\{ U_{i} \leq u \right\},
\end{align*}
and thus using (\ref{eqn:quantile_ecdf_lowerbd}), this implies
\begin{align}\label{eqn:quantile_sum_lowerbd}
\frac{1}{n} \sum_{i=1}^{n} \prr\left\{ U_{i} \leq \qnt_{\alpha}[\U_{n}] \right\} \geq \alpha.
\end{align}
It is typical to assume that the distribution function $F$ is right-continuous, i.e., that $F(u^{+}) = F(u)$ for all $u$, where we use the standard notation $F(u^{+}) \defeq \lim_{a \downarrow u} F(a)$ \citep[Ch.~1.4]{ash2000a}. When $F$ is continuous, then $F(u) = \prr\{ U \leq u \} = \prr\{ U < u \}$ for all $u$. Otherwise, we may have $\prr\{ U < u \} = F(u^{-}) < F(u)$, where $F(u^{-}) \defeq \lim_{a \uparrow u} F(a)$. In this case, it makes sense to discriminate between $F$ and $F^{-}$, where we denote $F^{-}(u) \defeq F(u^{-}) = \prr\{ U < u \}$. The empirical right-quantile function then is $\widehat{F}_{n}^{-}(u) \defeq (1/n) \sum_{i=1}^{n} I_{\{U_{i} < u\}}$. If we define analogous \term{right-quantiles} as
\begin{align*}
\qnt_{\alpha}^{-}[U] & \defeq \sup \left\{ u \in \RR: F^{-}(u) \leq \alpha \right\}\\
\qnt_{\alpha}^{-}[\U_{n}] & \defeq \sup \left\{ u \in \RR: \widehat{F}_{n}^{-}(u) \leq \alpha \right\},
\end{align*}
note that it follows immediately that $\widehat{F}_{n}^{-}(\qnt_{\alpha}^{-}[\U_{n}]) \leq \alpha$ for any sample $\U_{n}$, and thus taking expectation over the sample it follows that
\begin{align}\label{eqn:quantile_sum_upperbd}
\frac{1}{n}\sum_{i=1}^{n} \prr\left\{ U_{i} < \qnt_{\alpha}^{-}[\U_{n}] \right\} \leq \alpha.
\end{align}

It is often convenient to express empirical quantiles in terms of their order (sorted in either ascending or descending order). For ascending order, write $U_{(k,n)}$ to denote the $k$th-\term{smallest} element of $\U_{n}$. For descending order, write $U_{[k,n]}$ to denote the $k$th-\term{largest} element of $\U_{n}$. This notation implies that
\begin{align*}
U_{(1,n)} \leq U_{(2,n)} \leq \ldots \leq U_{(n,n)}\\
U_{[1,n]} \geq U_{[2,n]} \geq \ldots \geq U_{[n,n]}.
\end{align*}
Note that for any $1 \leq k \leq n$, the $k$th-smallest element is the $(n-k+1)$th-largest element. This leads us to a convenient algebraic relation between the quantities of interest, namely
\begin{align}
\label{eqn:quantile_order_relation}
\qnt_{\alpha}\left[ \U_{n} \right] & = U_{(\lceil n\alpha \rceil, n)} = U_{[n-\lceil n\alpha \rceil+1,n]}\\
\label{eqn:rightquantile_order_relation}
\qnt_{\alpha}^{-}\left[ \U_{n} \right] & = U_{(\lfloor n\alpha \rfloor+1, n)} = U_{[n-\lfloor n\alpha \rfloor,n]}
\end{align}
for any choice of $0 < \alpha < 1$.

\paragraph{Conformal prediction background}

The general-purpose methodology known as \term{conformal prediction} is concerned with quantifying and (to some extent) characterizing the uncertainty involved in making predictions based on data. This methodology is rooted in simple but very useful statistical principles, and we briefly describe some key points here. Of chief importance is the fact that all sorts of desirable properties are known for the empirical quantiles of ``exchangeable'' data. Letting $\U_{n} = (U_{1},\ldots,U_{n})$ be our data (a collection of random variables) for the moment, we say that these random variables are \term{exchangeable} if for any permutation $\pi$, the joint distribution of $U_{\pi(1)},\ldots,U_{\pi(n)}$ is the same as that of the original random variables. This clearly holds for iid data, but exchangeability is a weaker condition \citep{vovk2005ALRW,shafer2008a}. Recalling our preliminary look at quantiles, from (\ref{eqn:quantile_order_relation}), we see that for \emph{any} collection of data (regardless of exchangeability), for each $i \in [n]$ we can say
\begin{align}\label{eqn:quantile_order_iff}
U_{i} \leq \qnt_{\alpha}\left[\U_{n}\right] \iff U_{i} \leq U_{(\lceil n\alpha \rceil,n)}.
\end{align}
In general, if the nature of each $U_{i}$ differs with $i$, then the probability that this event occurs may differ greatly depending on the choice of $i$. However, when we have exchangeable data, it follows directly from the definition of exchangeability that for any $1 \leq k \leq n$, the event $\{U_{i} \leq U_{(k,n)}\}$ has the same probability for every $i \in [n]$. Taking this fact along with (\ref{eqn:quantile_sum_lowerbd}), it immediately follows that for any $0 < \alpha < 1$ and $i \in [n]$, we have
\begin{align*}
\prr\left\{ U_{i} \leq \qnt_{\alpha}\left[ \U_{n} \right] \right\} \geq \alpha.
\end{align*}
Furthermore, we can make even stronger statements when we eliminate ``ties,'' such that
\begin{align*}
\prr\left\{ U_{(1,n)} < U_{(2,n)} < \ldots < U_{(n,n)} \right\} = 1.
\end{align*}
If the joint distribution is continuous, then $\prr\{U_{i} = U_{j}\} = 0$ whenever $i \neq j$, so the random variables are almost surely distinct. Sorting exchangeable data then amounts to randomly placing indices in $n$ ``buckets,'' meaning that for any $i \in [n]$, there are $n$ mutually exclusive events with probability $\prr\{ U_{i} = U_{(k,n)} \} = 1/n$, one for each $k \in [n]$. It immediately follows that
\begin{align*}
\prr\left\{ U_{i} \leq \qnt_{\alpha}\left[\U_{n}\right] \right\} = \prr\left\{ U_{i} \leq U_{(\lceil n\alpha \rceil,n)} \right\} = \frac{\lceil n\alpha \rceil}{n} \leq \alpha + \frac{1}{n}.
\end{align*}
Organizing these facts, we have the following result.
\begin{lem}[Quantile lemma, on-sample]\label{lem:quantile_onsample}
Let the random variables $U_{1},\ldots,U_{n}$ be exchangeable. Then for any choice of $0 < \alpha < 1$ and $i \in [n]$, it follows that
\begin{align*}
\prr\left\{ U_{i} \leq \qnt_{\alpha}\left[ \U_{n} \right] \right\} \geq \alpha.
\end{align*}
Furthermore, if these $n$ random variables are almost surely distinct, we also have
\begin{align*}
\prr\left\{ U_{i} \leq \qnt_{\alpha}\left[ \U_{n} \right] \right\} \leq \alpha + \frac{1}{n}.
\end{align*}
\end{lem}

The preceding lemma is an important basic property that holds without making any special assumptions on the data beyond exchangeability. Moving forward, the key principles underlying conformal inference are concerned with how these empirical quantiles relate to \emph{new} data. To make this concrete, we consider making predictions about an as-yet unobserved $U$, having only observed the data set $\U_{n}$. Intuitively, if we compare the sorted elements of $\U_{n}$ and $\U_{n} \cup \{U\}$, there is minimal change to the ordering. When sorted in ascending order, only the points greater than $U$ see a change in their order, and even then it is just a simple shift of $+1$. In particular, the key observation is that if $U$ is among the $k$th-smallest elements of $\U_{n}$, then it is also among the $k$th-smallest elements of $\U_{n} \cup \{U\}$ (the converse also clearly holds; see Figure \ref{fig:schematic_sorting}). Written algebraically, the key fact is that for all $1 \leq k \leq n$, we have
\begin{align}
U \leq U_{(k,n)} \iff U \leq U_{(k,n+1)},
\end{align}
where $U_{(k,n+1)}$ denotes the $k$th-smallest element of the $n+1$ points in $\U_{n} \cup \{U\}$. Using this fact and the quantile relation (\ref{eqn:quantile_order_relation}), for any $0 < \beta < 1$ we have
\begin{align}\label{eqn:quantile_offsample_bridge}
\prr\left\{ U \leq \qnt_{\beta}\left[ \U_{n} \right] \right\} = \prr\left\{ U \leq U_{(\lceil n\beta \rceil, n)} \right\} = \prr\left\{ U \leq U_{(\lceil n\beta \rceil, n+1)} \right\}.
\end{align}
Now, in order to control the right-most probability using Lemma \ref{lem:quantile_onsample}, we need to link up $U_{(\lceil n\beta \rceil, n+1)}$ with $\qnt_{\alpha}[\U_{n} \cup \{U\}]$, where $\alpha$ is an arbitrary pre-fixed confidence level, and we can freely control $\beta$. This is easily done by setting $\beta$ such that
\begin{align*}
U_{(\lceil n\beta \rceil, n+1)} = U_{(\lceil (n+1)\alpha \rceil, n+1)} = \qnt_{\alpha}[\U_{n} \cup \{U\}],
\end{align*}
noting that the second equality follows from (\ref{eqn:quantile_order_relation}). This holds if we set $\beta = (1+1/n)\alpha$. Taking this back to (\ref{eqn:quantile_offsample_bridge}), we get
\begin{align*}
\prr\left\{ U \leq \qnt_{\beta}\left[ \U_{n} \right] \right\} = \prr\left\{ U \leq \qnt_{\alpha}[\U_{n} \cup \{U\}] \right\}.
\end{align*}
Applying Lemma \ref{lem:quantile_onsample} then immediately yields the following result.
\begin{lem}[Quantile lemma, off-sample]\label{lem:quantile_offsample}
Let the random variables $U_{1},\ldots,U_{n}$ and $U$ all be exchangeable. Then for any choice of $0 < \alpha < 1$, setting $\alpha_{n} \defeq (1+1/n)\alpha$, it follows that
\begin{align*}
\prr\left\{ U \leq \qnt_{\alpha_{n}}\left[ \U_{n} \right] \right\} \geq \alpha.
\end{align*}
Furthermore, if these $n+1$ random variables are almost surely distinct, we also have
\begin{align*}
\prr\left\{ U \leq \qnt_{\alpha_{n}}\left[ \U_{n} \right] \right\} \leq \alpha + \frac{1}{n+1}.
\end{align*}
\end{lem}
\noindent Unlike Lemma \ref{lem:quantile_onsample}, which only dealt with properties of the sample $\U_{n}$, here Lemma \ref{lem:quantile_offsample} allows us to construct confidence intervals for as-yet unobserved $U$, based on only the data $\U_{n}$, under equally weak assumptions on the underlying data distribution. This is a very important conceptual difference. In the context of conformal prediction, the lower bound in the preceding lemma proves the \term{validity} of the $\alpha$-level prediction interval for $U$ constructed using $\qnt_{\alpha_{n}}\left[ \U_{n} \right]$, whereas the upper bound proves that the predictions are well-\term{calibrated}.\footnote{See \citet{vovk2005ALRW} or \citet{shafer2008a} for more background on these terms and concepts.} We remark that these basic properties of quantiles are well-known in the literature.\footnote{For example, see the supplementary materials of \citet{tibshirani2020a} and \citet{romano2020a}.}

\begin{figure}[t]
\centering
\includegraphics[width=0.4\textwidth]{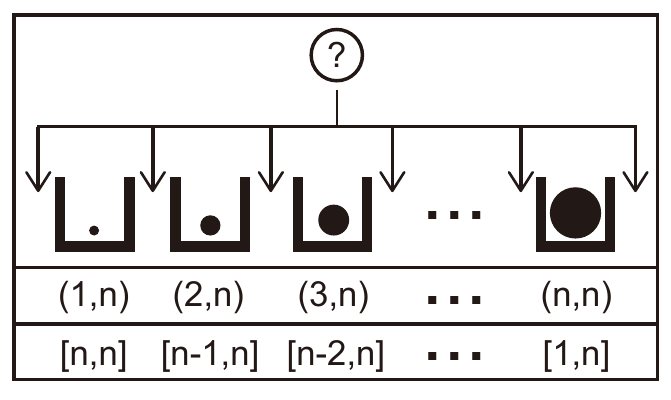}
\caption{A schematic for illustrating the simple ideas underlying conformal inference. The items in buckets (corresponding to the $\U_{n}=(U_{1},\ldots,U_{n})$ elements) have been sorted from left to right in ascending order, and the pairs $(k,n)$ and $[k,n]$ respectively denote the $k$th-smallest and $k$th-largest element of $n$ elements total, and the question mark represents a new point $U$. Visual inspection makes it clear that both $U \leq U_{(k,n)} \iff U \leq U_{(k,n+1)}$ and $U \geq U_{[k,n]} \iff U \leq U_{[k,n+1]}$ are valid statements for any $1 \leq k \leq n$.}
\label{fig:schematic_sorting}
\end{figure}

\paragraph{Conformal regression}

The most common application of the conformal prediction methodology is to ``regression'' problems, where our observed data take the form of $(\text{input},\text{response})$ pairs $Z_{i} = (X_{i},Y_{i})$. The usual setup has some underlying algorithm $\algo$ which, given a dataset $\Z_{n} = (Z_{1},\ldots,Z_{n})$ returns a predictor $\widehat{h}_{n} = \algo(\Z_{n})$ such that hopefully $\widehat{h}_{n}(X) \approx Y$ is an accurate approximation given a new $(X,Y)$ pair. The traditional approach involves constructing a ``score'' function $s(h;z)$ to evaluate the quality of any predictor $h$, and to consider \term{non-conformity scores} defined for each $z \in \ZZ$ by
\begin{align*}
S^{\prime}(z) & \defeq s(\algo(\Z_{n} \cup \{z\});z)\\
S_{i}(z) & \defeq s(\algo(\Z_{n} \cup \{z\});Z_{i}).
\end{align*}
Both $S^{\prime}(z)$ and $S_{i}(z)$ are used to evaluate the on-sample quality of the predictor returned by $\algo$ after being fed the same sample (which depends on the variable $z$). The only difference is the point at which the score is evaluated. Denoting the list of scores $\S_{n}(z) \defeq (S_{1}(z),\ldots,S_{n}(z))$, the traditional approach involves constructing a prediction set of the form
\begin{align}\label{eqn:cchat_cp_full}
\cchat_{\alpha}(x) \defeq \left\{ y \in \RR: S^{\prime}(x,y) \leq \qnt_{1-\alpha}\left[ \S_{n}(x,y) \cup \{\infty\} \right] \right\}.
\end{align}
Since both $S^{\prime}(\cdot)$ and the $S_{i}(\cdot)$ for each $i \in [n]$ can be evaluated based on the original $n$-sized sample, in \emph{principle} it is always possible to check the condition characterizing the set defined in (\ref{eqn:cchat_cp_full}). In practice, however, since doing so requires a full re-running of $\algo$ for each change to $x$ on the left-hand side and each candidate $y$ being checked in the set condition, this should be considered an ideal, intractable quantity. Utilizing the basic properties of empirical quantiles as discussed earlier, it can be shown that $\cchat_{\alpha}(X)$ as constructed in (\ref{eqn:cchat_cp_full}) provides us with a valid and well-calibrated predictive set for the response $Y$, as summarized in the following theorem.
\begin{thm}[Conformal regression]\label{thm:cp_regression_full}
Let the data $Z_{1},\ldots,Z_{n}$ and $Z$ all be exchangeable, and let $\cchat_{\alpha}(\cdot)$ be as in (\ref{eqn:cchat_cp_full}), for $0 < \alpha < 1$. It follows that
\begin{align*}
\prr\left\{ Y \in \cchat_{\alpha}(X)  \right\} \geq 1-\alpha.
\end{align*}
Furthermore, when the realizations of the non-conformity scores $S_{1}(Z),\ldots,S_{n}(Z)$ and $S^{\prime}(Z)$ are almost surely distinct, we also have that
\begin{align*}
\prr\left\{ Y \in \cchat_{\alpha}(X)  \right\} \leq 1-\alpha + \frac{1}{n+1}.
\end{align*}
\end{thm}
\noindent As with the quantile lemmas discussed earlier, this result is remarkable due to its generality. Nothing special is assumed about the score function $s(h;z)$, the algorithm $\algo$, or the data-generating process beyond exchangeability. The lower bound on the coverage probability here is a foundational result of conformal inference; see \citet{vovk2005ALRW}. The upper bound is due to \citet{lei2018a}. The ``almost surely distinct'' requirement can be easily removed if we have some randomized sub-routine for breaking ties. Exchangeability assumptions can be further weakened using a slightly different procedure; see \citet{tibshirani2020a}.

\paragraph{Split conformal regression}

In order to make the conformal prediction approach more practical, the most obvious strategy is to simply split the sample. Given $\Z_{n} = (Z_{1},\ldots,Z_{n})$ to start with, split the index $[n] = \II_{\TR} \cup \II_{\CP}$, with $\II_{\TR}$ for training, and $\II_{\CP}$ for the conformal prediction sub-routine. Write $\Z_{\TR}$ and $\Z_{\CP}$ for $\Z_{n}$ sub-indexed by $\II_{\TR}$ and $\II_{\CP}$ respectively (order of sub-lists is assumed random). The \term{split conformal prediction} approach then uses a simple non-conformity score
\begin{align*}
S(x,y) \defeq s(\algo(\Z_{\TR});x,y)
\end{align*}
which will then be evaluated as $S(X_{i},Y_{i})$ for each $i \in \II_{\CP}$. To denote these scores compactly, we write $\S_{\CP} \defeq \{S(X_{i},Y_{i}): i \in \II_{\CP}\}$. Note that if the data is exchangeable, then so long as the algorithm $\algo$ does not pay attention to the data index (or randomly shuffles the data points before execution), it immediately follows that the non-conformity scores $\S_{\CP}$ and $S(X,Y)$ evaluated at a new point $(X,Y)$ are also exchangeable. Recalling the off-sample quantile property given in Lemma \ref{lem:quantile_offsample}, if we set $\alpha_{\CP} \defeq (1+1/|\II_{\CP}|)(1-\alpha)$, then dealing with ties appropriately we obtain validity and calibration guarantees for data-driven prediction of as-yet unobserved non-conformity scores:
\begin{align*}
1-\alpha \leq \prr\left\{ S(X,Y) \leq \qnt_{\alpha_{\CP}}\left[ \S_{\CP} \right] \right\} \leq 1-\alpha + \frac{1}{|\II_{\CP}|+1}.
\end{align*}
One can then readily reverse-engineer a prediction set for the response $Y$ by defining
\begin{align}\label{eqn:cchat_cp_split}
\cchat_{\alpha}(x) \defeq \left\{ y \in \RR: S(x,y) \leq \qnt_{\alpha_{\CP}}\left[ \S_{\CP} \right] \right\},
\end{align}
to which the validity and calibration guarantees apply as-is. These facts are summarized in the following theorem.
\begin{thm}[Split conformal regression]\label{thm:cp_regression_split}
Let the data $Z_{1},\ldots,Z_{n}$ and $Z$ all be exchangeable, and let $\cchat_{\alpha}(\cdot)$ be as in (\ref{eqn:cchat_cp_split}), for $0 < \alpha < 1$. It follows that
\begin{align*}
\prr\left\{ Y \in \cchat_{\alpha}(X)  \right\} \geq 1-\alpha.
\end{align*}
Furthermore, when the realizations of the non-conformity scores $S(X_{i},Y_{i})$ (for $i \in \II_{\CP}$) and $S(X,Y)$ are almost surely distinct, we also have that
\begin{align*}
\prr\left\{ Y \in \cchat_{\alpha}(X)  \right\} \leq 1-\alpha + \frac{1}{|\II_{\CP}|+1}.
\end{align*}
\end{thm}
\noindent Note the immediate computational advantage: we only have to run $\algo$ on $\Z_{\TR}$ once, in contrast with the full conformal prediction approach that requires re-running $\algo$ for every candidate to be checked for inclusion in the set. This simple trick to make practical applications of conformal prediction feasible is well-known in the literature; see \citet{papadopoulos2002a,vovk2005ALRW} for early work, and \citet{lei2018a} for more recent analysis and discussion. The price for this efficiency is weaker calibration guarantees (since $|\II_{\CP}| < n$), and more coarse-grained prediction sets. The conformal prediction method can be applied to a much richer class of prediction tasks than just constructing predictive sets. One lucid example is the problem of constructing valid predictive distributions; see \citet{vovk2018a} and \citet{vovk2019a}.

\section{Conformal performance prediction}\label{sec:cpp}

As introduced in section \ref{sec:intro}, our purpose is to investigate the use of conformal prediction techniques to provide answers to a new class of questions regarding the performance of learning algorithms. We call our general approach \term{conformal performance prediction} (CPP). The key elements of our formulation are as follows:
\begin{itemize}
\item Desired confidence parameter $0 < \alpha < 1$.
\item Random sample $\Z_{n} = (Z_{1},\ldots,Z_{n})$, used for both training and calibration.
\item Learning algorithm $\algo$. When passed a sample, $\algo$ returns an element of $\HH$.
\item Loss function $\loss(h;z)$, defined for all $h \in \HH$ and $z \in \ZZ$.
\item As-yet unobserved random data point $Z$.
\item As-yet unobserved random sample $\Z$ (size to be specified later).
\item Predictive set $\widehat{\EE}_{\alpha}$, which aims to include as-yet unseen losses $(1-\alpha) \times 100\%$ of the time.
\end{itemize}
Given a one-sentence summary, CPP is the construction of predictive sets $\widehat{\EE}_{\alpha}$ for the \emph{loss to be incurred} by algorithm $\algo$, which are provably valid under exchangeable data. The framework allows for several variations in terms of how we define the loss to be incurred, and how we construct the predictive set $\widehat{\EE}_{\alpha}$. These details will be described in the following sub-sections.

\subsection{Warm-up: CPP for candidates}\label{sec:cpp_candidate}

To begin, the simplest notion of a ``loss to be incurred by $\algo$'' is that of the loss incurred by the candidate chosen by $\algo$ given a \emph{particular} sample. That is, we make a disjoint partition of the sample $\Z_{n}$ using sub-indices $[n] = \II_{\TR} \cup \II_{\CP}$, and denoting $\algo_{\TR} \defeq \algo(\Z_{\TR})$, we desire a predictive set such that $\loss(\algo_{\TR};Z) \in \widehat{\EE}_{\alpha}$ with probability at least $1-\alpha$, over the random draw of both $\Z_{n}$ and $Z$. This is essentially nothing more than traditional conformal prediction with the focused shifted to the scoring mechanism. As such, we can easily utilize the same basic principles underlying split conformal prediction, as we describe below.

If all we want is a one-sided interval, note that if we set $\alpha_{\CP} \defeq (1+1/|\II_{\CP}|)(1-\alpha)$, then up to exchangeability assumptions, Lemma \ref{lem:quantile_offsample} immediately implies
\begin{align}\label{eqn:cpp_candidate_onesided_lb}
\prr\left\{ \loss(\algo_{\TR};Z) \leq \qnt_{\alpha_{\CP}}\left[ \losses_{\CP} \right] \right\} \geq 1-\alpha,
\end{align}
where we write $\losses_{\CP} \defeq \{ \loss(\algo_{\TR};Z_{i}) : i \in \II_{\CP} \}$ for the losses incurred by $\algo_{\TR}$ on the calibration data. Regarding exchangeability, for (\ref{eqn:cpp_candidate_onesided_lb}) to hold, the losses $\loss(\algo_{\TR};Z_{i})$ (for $i \in \II_{\CP}$) and $\loss(\algo_{\TR};Z)$ must all be exchangeable. Since $\II_{\TR}$ and $\II_{\CP}$ are disjoint, using the pointwise form of the losses, this follows immediately if the data $\Z_{\CP}$ and $Z$ are exchangeable, regardless of the nature of algorithm $\algo$. Furthermore, if the losses are almost surely distinct over the random draw of $\Z_{n}$, then Lemma \ref{lem:quantile_offsample} also implies
\begin{align*}
\prr\left\{ \loss(\algo_{\TR};Z) \leq \qnt_{\alpha_{\CP}}\left[ \losses_{\CP} \right] \right\} \leq 1-\alpha + \frac{1}{|\II_{\CP}|+1}.
\end{align*}

In practice, it may be more useful to have two-sided prediction intervals, and this can be easily achieved using a straightforward extension of split conformal prediction. With desired confidence parameter $\alpha$ in hand, set the adjusted low/high quantiles to be
\begin{align}\label{eqn:cpp_candidate_alphas}
\alpha_{\textup{lo}} \defeq \frac{\alpha}{2} - \frac{(1-\alpha/2)}{|\II_{\CP}|}, \quad \alpha_{\textup{hi}} \defeq \left(1+\frac{1}{|\II_{\CP}|}\right)\left(1-\frac{\alpha}{2}\right)
\end{align}
and then create a predictive interval as
\begin{align}\label{eqn:cpp_candidate_predset}
\widehat{\EE}_{\alpha}(\algo_{\TR}) \defeq \left\{ u \in \RR: \qnt_{\alpha_{\textup{lo}}}\left[ \losses_{\CP} \right] \leq u \leq \qnt_{\alpha_{\textup{hi}}}\left[ \losses_{\CP} \right] \right\},
\end{align}
where we write $\losses_{\CP} \defeq \{ \loss(\algo_{\TR};Z_{i}) : i \in \II_{\CP} \}$ for the losses incurred by $\algo_{\TR}$ on the calibration data. The basic procedure is summarized in Algorithm \ref{algo:cpp_candidate}. Using the key quantile lemma for off-sample prediction, we can readily prove the desired validity of this CPP set (the proof of this, and all subsequent results, is given in appendix \ref{sec:apdx_proofs}).
\begin{thm}[Candidate CPP]\label{thm:cpp_candidate}
Let the data $Z_{1},\ldots,Z_{n}$ and $Z$ all be exchangeable. Then, setting $\widehat{\EE}_{\alpha}$ as in (\ref{eqn:cpp_candidate_predset}), for any choice of $0 < \alpha < 1$, we have that
\begin{align*}
\prr\left\{ \loss(\algo_{\TR};Z) \in \widehat{\EE}_{\alpha}(\algo_{\TR}) \right\} \geq 1-\alpha.
\end{align*}
Furthermore, if the losses $\loss(\algo_{\TR};Z_{i})$ (for $i \in \II_{\CP}$) and $\loss(\algo_{\TR};Z)$ are almost surely distinct, then we also have
\begin{align*}
\prr\left\{ \loss(\algo_{\TR};Z) \in \widehat{\EE}_{\alpha}(\algo_{\TR}) \right\} \leq 1-\alpha + \frac{2}{|\II_{\CP}|+1}.
\end{align*}
\end{thm}
\begin{rmk}[CPP versus traditional CP]
This simple CPP set construction is very closely related to traditional (split) conformal prediction. In typical conformal prediction, one assigns ``scores'' to the predictions made by the base predictor and constructs valid prediction intervals for an as-yet unobserved score, from which one then reverse-engineers a prediction interval for the as-yet unobserved label. If one replaces the base predictor with $\algo_{\TR}$ (which need not ``predict'' anything), replaces the scores with generic losses, and constructs a two-sided prediction interval instead of a one-sided interval, it amounts to precisely the CPP procedure described above. In this sense, since there is no need to reverse-engineer a prediction set for new labels, this CPP procedure can be considered even simpler than traditional conformal prediction.\hfill$\blacksquare$
\end{rmk}
\begin{rmk}[Discrete versus continuous losses]
Note that the validity guarantee of Theorem \ref{thm:cpp_candidate} holds for all manners of loss functions, including cases where $\loss(h;Z)$ is a discrete random variable for all $h \in \HH$ (e.g., where the loss of interest is the zero-one error). As the term ``loss'' implies, we are assuming that $\loss(h;z)$ takes numerical values, understood as a penalty incurred by $h$ at $z$, and that all else equal, smaller losses are better. As such, even when $\loss(h;Z)$ is a discrete random variable, there is value in constructing CPP sets. Instead of answering questions like
\begin{center}
\textit{``Given an input, how confident are we in what the true label will be?''}
\end{center}
which are the domain of traditional conformal prediction, we can answer questions like
\begin{center}
\textit{``Given a trained classifier, how confident are we that it will\\correctly classify a new instance?''}
\end{center}
making particular use of the \emph{boundary} of the CPP interval $\widehat{\EE}_{\alpha}$. On the other hand, when the loss is continuous (e.g., squared error, logistic loss, etc.), the calibration guarantee of Theorem \ref{thm:cpp_candidate} also can be brought to bear on the problem, meaning that both the boundary and the \emph{width} of the CPP interval $\widehat{\EE}_{\alpha}$ will be of interest, particularly when comparing different algorithms.\hfill$\blacksquare$
\end{rmk}

\begin{algorithm}[t!]
\caption{Candidate CPP.}
\label{algo:cpp_candidate}
\begin{algorithmic}
\State \textbf{inputs:} algorithm $\algo$, data $\Z_{n}$, level $\alpha$.
\medskip
\State Split the data index $\displaystyle [n] = \II_{\TR} \cup \II_{\CP}$, typically with $\displaystyle |\II_{\TR}| = |\II_{\CP}|=\lfloor n/2 \rfloor$.
\medskip
\State Run algorithm on training data to obtain a candidate $\displaystyle \algo_{\TR} = \algo(\Z_{\TR})$.
\medskip
\State Evaluate this candidate and compute losses $\displaystyle \losses_{\CP} = \left\{ \loss(\algo_{\TR};Z_{i}): i \in \II_{\CP} \right\}$.
\medskip
\State Compute inflated and deflated quantiles $\alpha_{\text{hi}}$ and $\alpha_{\text{lo}}$ as in (\ref{eqn:cpp_candidate_alphas}).
\medskip
\State \textbf{return:} $\displaystyle \widehat{\EE}_{\alpha}(\algo_{\TR}) = \left\{ u \in \RR: \qnt_{\alpha_{\text{lo}}}\left[\losses_{\CP}\right] \leq u \leq \qnt_{\alpha_{\text{hi}}}\left[\losses_{\CP}\right] \right\}$.
\end{algorithmic}
\end{algorithm}

\paragraph{Aside: beware of naive swapping of losses}

One may be tempted to naively just ``swap'' the prediction-oriented non-conformity scores for generic losses in the full conformal prediction procedure. For example, a simple modification to the proof of Theorem \ref{thm:cp_regression_full} allows one to derive results such as
\begin{align}
1-\alpha \leq \prr\left\{ \loss(\algo_{n}(Z);Z) \leq \qnt_{1-\alpha}\left[ \losses_{n}(Z) \cup \{\infty\}\right] \right\} \leq 1-\alpha + \frac{1}{n+1},
\end{align}
where $\algo_{n}(z) \defeq \algo(\Z_{n} \cup \{z\})$ and $\losses_{n}(z) \defeq \{\loss(\algo_{n}(z);Z_{i}) : i \in [n]\}$ for all $z \in \ZZ$. However, a moment's thought shows that this does not achieve what we desire. Leaving the computational intractability aside, this is conceptually misguided since as we need to have $Z$ in order to compute the bound $\qnt_{1-\alpha}\left[ \losses_{n}(Z) \cup \{\infty\}\right]$, we can just compute $\loss(\algo_{n}(Z);Z)$ directly, and thus spending resources on prediction for this quantity becomes meaningless.

\subsection{CPP for algorithms}\label{sec:cpp_algo}

Next, we consider modifications to the CPP approach described in the previous section, in order to gain new insights into the uncertainty that arises due to the \emph{interaction} of the algorithm $\algo$ and the data it uses for training. The most direct entry point to this problem is to simply split up the data in order to obtain multiple training sets. This introduces a higher cost in terms of data needed for training, but allows the user to gain new insights into the uncertainty of the learning process, while enjoying formal guarantees of validity and calibration. Details are provided in the following three sub-sections, with the core procedures summarized in Algorithms \ref{algo:cpp_algo_zfree} and \ref{algo:cpp_algo_zmod}.

\subsubsection{$Z$-free CPP sets}\label{sec:cpp_algo_zfree}

Assuming our full data set is $\Z_{n} = (Z_{1},\ldots,Z_{n})$, we break the data into subsets for training ($\TR$) and individual points for evaluation ($\EV$) as follows: split $[n] = \II_{\TR} \cup \II_{\EV}$, using sub-indices
\begin{align}\label{eqn:cpp_algo_zfree_partition}
\II_{\TR} = \bigcup_{j \in \II_{\EV}} \II_{\TR}^{(j)}.
\end{align}
This allows us to construct tuples of the form $(Z_{j},\Z_{\TR}^{(j)})$, where we denote the training subsets by $\Z_{\TR}^{(j)} \defeq \{Z_{i}: i \in \II_{\TR}^{(j)}\}$, for each $j \in \II_{\EV}$. These are then passed through $\algo$ and $\loss$ to obtain the set of losses
\begin{align*}
\losses_{\EV} \defeq \left\{ \loss(\algo_{\TR}^{(j)};Z_{j}): j \in \II_{\EV} \right\},
\end{align*}
where we denote $\algo_{\TR}^{(j)} \defeq \algo(\Z_{\TR}^{(j)})$. Using these losses, we set quantile levels to
\begin{align}\label{eqn:cpp_algo_zfree_alphas}
\alpha_{\textup{lo}} \defeq \frac{\alpha}{2}-\frac{1-\alpha/2}{|\II_{\EV}|}, \quad \alpha_{\textup{hi}} \defeq \left(1+\frac{1}{|\II_{\EV}|}\right)\left(1-\frac{\alpha}{2}\right)
\end{align}
and construct a CPP set as
\begin{align}\label{eqn:cpp_algo_zfree_predset}
\widehat{\EE}_{\alpha}(\algo) \defeq \left\{u \in \RR: \qnt_{\alpha_{\textup{lo}}}[\losses_{\EV}] \leq u \leq \qnt_{\alpha_{\textup{hi}}}[\losses_{\EV}] \right\}.
\end{align}
Clearly, this is reminiscent of the construction (\ref{eqn:cpp_candidate_predset}), with the key difference being that for each point in the evaluation set $\losses_{\EV}$, the algorithm $\algo$ is being run on a different training set. Since the prediction interval does not depend on the new point $Z$, we refer to it as being \term{$Z$-free}. In order to ensure that these losses are exchangeable, we require that the algorithm $\algo$ not pay attention to the order of the data points it is passed, a property often called \term{symmetry} \citep{bousquet2002a}. That is, for any sample size $m > 0$ and permutation $\pi$, we require that
\begin{align}\label{eqn:algo_symmetry}
\algo(Z_{1},\ldots,Z_{m}) \text{ and } \algo(Z_{\pi(1)},\ldots,Z_{\pi(m)}) \text{ have the same distribution over } \HH.
\end{align}
Even if $\algo$ operates sequentially on its inputs, as long as it randomly shuffles these inputs before going to work, the symmetry property (\ref{eqn:algo_symmetry}) holds. With this established, we can readily use the arguments made in the proof of Theorem \ref{thm:cpp_candidate} to derive an analogous guarantee here.
\begin{thm}[Algorithm CPP: $Z$-free case]\label{thm:cpp_algo_zfree}
Given the CPP set $\widehat{\EE}_{\alpha}$ as defined in (\ref{eqn:cpp_algo_zfree_predset}), let the partition of $\II_{\TR}$ in (\ref{eqn:cpp_algo_zfree_partition}) be disjoint, with all subsets $\II_{\TR}^{(j)}$ having equal cardinality,\footnote{Here and in what follows, for simplicity we assume that when the number of samples does not divide evenly, the randomly selected remaining points are discarded.} say $m = |\II_{\TR}^{(j)}|$ for all $j \in \II_{\EV}$. Let $\Z$ denote a fresh sample of size $m$. When the random variables $Z_{1},\ldots,Z_{n}$, $Z$, and $\Z$ are all exchangeable, and the algorithm $\algo$ is symmetric in the sense of (\ref{eqn:algo_symmetry}), we have
\begin{align*}
\prr\left\{ \loss(\algo(\Z);Z) \in \widehat{\EE}_{\alpha}(\algo) \right\} \geq 1-\alpha.
\end{align*}
Furthermore, when the losses $\loss(\algo_{\TR}^{(j)};Z_{j})$ (for $j \in \II_{\EV}$) and $\loss(\algo(\Z);Z)$ are all almost surely distinct, we also have that
\begin{align*}
\prr\left\{ \loss(\algo(\Z);Z) \in \widehat{\EE}_{\alpha}(\algo) \right\} \leq 1-\alpha + \frac{2}{|\II_{\EV}|+1}.
\end{align*}
\end{thm}
\begin{rmk}
We reinforce the point that the random event considered in Theorem \ref{thm:cpp_algo_zfree} considers the random draw of an as-yet unobserved data sample $\Z$, in addition to the new point $Z$. This is the critical difference between Theorems \ref{thm:cpp_algo_zfree} and \ref{thm:cpp_candidate}; the new sample $\Z$ plays no role in Theorem \ref{thm:cpp_candidate}, and the uncertainty is only due to the random draw of $Z$. On the other hand, Theorem \ref{thm:cpp_algo_zfree} lets us use a computable quantity $\widehat{\EE}_{\alpha}$, depending only on $\Z_{n}$, to make predictive statements about the future performance of the algorithm $\algo$ when given another data sample for training. For example, this could be used to make predictions about the performance of competing algorithms on an independent trial covering the full train-test cycle.\hfill$\blacksquare$
\end{rmk}

\begin{algorithm}[t!]
\caption{Algorithm CPP ($Z$-free).}
\label{algo:cpp_algo_zfree}
\begin{algorithmic}
\State \textbf{inputs:} algorithm $\algo$, data $\Z_{n}$, level $\alpha$, set size $k$.
\medskip
\State Split the data index $\displaystyle [n] = \II_{\TR} \cup \II_{\EV}$, with $|\II_{\EV}|=k$.
\medskip
\State Further partition $\II_{\TR}$ as in (\ref{eqn:cpp_algo_zfree_partition}), yielding $\displaystyle |\II_{\TR}^{(j)}| = \lfloor (n-k)/k \rfloor$ for each training subset.
\medskip
\State Run algorithm on training data to obtain $k$ candidates: $\displaystyle \left\{ \algo_{\TR}^{(j)} = \algo(\Z_{\TR}^{(j)}): j \in \II_{\EV} \right\}$.
\medskip
\State Evaluate these candidates and compute losses $\displaystyle \losses_{\EV} = \left\{ \loss(\algo_{\TR}^{(j)};Z_{j}): j \in \II_{\EV} \right\}$.
\medskip
\State Compute inflated and deflated quantiles $\alpha_{\text{hi}}$ and $\alpha_{\text{lo}}$ as in (\ref{eqn:cpp_algo_zfree_alphas}).
\medskip
\State \textbf{return:} $\displaystyle \widehat{\EE}_{\alpha}(\algo) = \left\{ u \in \RR: \qnt_{\alpha_{\text{lo}}}\left[\losses_{\EV}\right] \leq u \leq \qnt_{\alpha_{\text{hi}}}\left[\losses_{\EV}\right] \right\}$.
\end{algorithmic}
\end{algorithm}

\begin{algorithm}[t!]
\caption{Algorithm CPP ($Z$-modulated).}
\label{algo:cpp_algo_zmod}
\begin{algorithmic}
\State \textbf{inputs:} algorithm $\algo$, data $\Z_{n}$, level $\alpha$, choice of sub-routine $\texttt{sub}$, set size $k$.
\medskip
\State Split the data index $[n] = \II_{\TR} \cup \II_{\CP}$, with $|\II_{\CP}| = |\II_{\CP}| = \lfloor n/2 \rfloor$.
\medskip
\State Further partition $\II_{\TR}$ and $\II_{\CP}$ as in (\ref{eqn:cpp_algo_zmod_fixed_partition}), with $|\II_{\TR}^{\prime}|=|\II_{\CP}^{\prime}|=k$.
\medskip
\State Run algorithm, yielding $2k$ candidates: $\displaystyle \left\{ \algo_{\TR}^{(j)}: j \in \II_{\TR}^{\prime} \right\}$ and $\displaystyle \left\{ \algo_{\CP}^{(j)}: j \in \II_{\CP}^{\prime} \right\}$.
\medskip
\State Construct specialized $(\text{input},\text{response})$ pairs $\displaystyle \mv{D}_{\TR}^{\prime} = \left\{ (Z_{j}, \loss(\algo_{\TR}^{(j)};Z_{j})): j \in \II_{\TR}^{\prime} \right\}$.
\medskip
\State Compute inflated quantile level $\alpha_{\CP} = \min\{1, (1-\alpha)(1+1/k) \}$.
\medskip
\If{$\texttt{sub}=\texttt{Regression}$}
\medskip
\State Get location predictor $\displaystyle \widehat{f}_{\TR} = \texttt{Regression}\left[\mv{D}_{\TR}^{\prime}\right]$.
\medskip
\State Score this predictor with $\S_{\CP}^{\prime}$ as in (\ref{eqn:cpp_algo_zmod_fixed_scores}).
\medskip
\State \textbf{return:} $\widehat{\EE}_{\alpha}(z;\algo) = \widehat{f}_{\TR}(z) \pm \qnt_{\alpha_{\CP}}\left[\S_{\CP}^{\prime}\right]$
\EndIf
\medskip
\If{$\texttt{sub}=\texttt{QuantReg}$}
\medskip
\State Get low/high quantile predictors $\displaystyle (\widehat{q}_{\text{lo}},\widehat{q}_{\text{hi}}) = \texttt{QuantReg}\left[\mv{D}_{\TR}^{\prime}\right]$.
\medskip
\State Score these predictors with $\S_{\CP}^{\prime\prime}$ as in (\ref{eqn:cpp_algo_zmod_variable_scores}).
\medskip
\State \textbf{return:} $\widehat{\EE}_{\alpha}(z;\algo) = \left[ \widehat{q}_{\text{lo}}(z) - \qnt_{\alpha_{\CP}}\left[\S_{\CP}^{\prime\prime}\right], \widehat{q}_{\text{hi}}(z) + \qnt_{\alpha_{\CP}}\left[\S_{\CP}^{\prime\prime}\right] \right]$
\EndIf
\end{algorithmic}
\end{algorithm}

\subsubsection{$Z$-modulated CPP sets}\label{sec:cpp_algo_zmod_fixed}

Next, we consider the task of making predictive statements about the performance of a learning algorithm being fed an as-yet unobserved training sample, where the evaluation is conditional on a particular data point that we have access to. To realize such a procedure, we apply a conformal regression ``wrapper'' to the data-generating process of $Z \mapsto \loss(\algo(\Z);Z)$. In analogy with traditional regression problems, $Z$ is the ``input,'' $\loss(\algo(\Z);Z)$ is the ``response,'' and the ``noise'' arises due to the interaction of $\algo$ with the random sample $\Z$.

Let us spell this out a bit more explicitly. Given data $\Z_{n}$ to start with, much like the case of candidate-centric CPP, we split the index into $[n] = \II_{\TR} \cup \II_{\CP}$, a set for training, and a set for conformal prediction, assumed to be disjoint. Since we need to establish a correspondence between evaluation points and the losses incurred by algorithm $\algo$, we further break down these indices as
\begin{align}\label{eqn:cpp_algo_zmod_fixed_partition}
\II_{\TR} = \II_{\TR}^{\prime} \cup \left( \bigcup_{j \in \II_{\TR}^{\prime}} \II_{\TR}^{(j)} \right),%
\quad%
\II_{\CP} = \II_{\CP}^{\prime} \cup \left( \bigcup_{j \in \II_{\CP}^{\prime}} \II_{\CP}^{(j)} \right).
\end{align}
The training data is used to pin down this correspondence using any regression sub-routine,\footnote{Traditionally, the term ``regression'' is often taken to mean ``regression by least squares,'' and thus regardless of model assumptions such as linearity, the traditional setting amounts to an approximation of the conditional expectation mapping $z \mapsto \exx [\loss(\algo(\Z);Z) \cond Z=z]$. Classical alternatives include least absolute deviations, which approximates the conditional median $z \mapsto \qnt_{1/2}[\loss(\algo(\Z);Z) \cond Z=z]$, and more general robust regression techniques. Here we take a broad view of regression methods, allowing any procedure which approximates the central tendency of $\loss(\algo(\Z);Z)$, conditioned on the event $\{Z=z\}$.} denoted here by $\texttt{Regression}[\cdot]$. We pass $(\text{input},\text{response})$ pairs to this sub-routine as
\begin{align*}
\widehat{f}_{\TR} = \texttt{Regression}\left[ \left\{ (Z_{j},\loss(\algo_{\TR}^{(j)};Z_{j})): j \in \II_{\TR}^{\prime} \right\} \right],
\end{align*}
where once again we denote $\algo_{\TR}^{(j)} \defeq \algo(\Z_{\TR}^{(j)})$ and $\Z_{\TR}^{(j)} \defeq \{Z_{i} \in \II_{\TR}^{(j)}\}$. The remaining data points are used to apply a conformal regression wrapper to this predictor $\widehat{f}_{\TR}$, using absolute residuals as non-conformity scores. That is, we define
\begin{align}\label{eqn:cpp_algo_zmod_fixed_scores}
\S_{\CP}^{\prime} \defeq \left\{ |\widehat{f}_{\TR}(Z_{j})-\loss(\algo_{\CP}^{(j)};Z_{j})|: j \in \II_{\CP}^{\prime} \right\},
\end{align}
where $\algo_{\CP}^{(j)}$ is defined analogously to $\algo_{\TR}^{(j)}$, and then proceed to construct CPP intervals as
\begin{align}\label{eqn:cpp_algo_zmod_fixed_predset}
\widehat{\EE}_{\alpha}(z;\algo) \defeq \left\{ u \in \RR: \widehat{f}_{\TR}(z)-\qnt_{\alpha_{\CP}}\left[ \S_{\CP}^{\prime} \right] \leq u \leq \widehat{f}_{\TR}(z)+\qnt_{\alpha_{\CP}}\left[ \S_{\CP}^{\prime} \right] \right\},
\end{align}
where $\alpha_{\CP} \defeq (1+1/|\II_{\CP}^{\prime}|)(1-\alpha)$. Since this CPP interval can be constructed using a new $Z$ to make predictions about the performance of $\algo$ on a fresh sample, we refer to it as being \term{$Z$-modulated}. As long as we have a symmetric algorithm, we can readily apply standard results for split conformal regression (e.g., Theorem \ref{thm:cp_regression_split}) to our setting, allowing us to obtain validity and calibration guarantees as follows.
\begin{thm}[Algorithm CPP: $Z$-modulated, fixed-width case]\label{thm:cpp_algo_zmod_fixed}
Given the CPP set $\widehat{\EE}_{\alpha}$ as defined in (\ref{eqn:cpp_algo_zmod_fixed_predset}), let the partition of $\II_{\TR}$ and $\II_{\CP}$ in (\ref{eqn:cpp_algo_zmod_fixed_partition}) be disjoint, with all subsets $\II_{\TR}^{(j)}$ and $\II_{\CP}^{(j)}$ having equal cardinality, say $m = |\II_{\TR}^{(j)}| = |\II_{\CP}^{(j)}|$ for all $j$. Let $\Z$ denote a fresh sample of size $m$. When the random variables $Z_{1},\ldots,Z_{n}$, $Z$, and $\Z$ are all exchangeable, and the algorithm $\algo$ is symmetric in the sense of (\ref{eqn:algo_symmetry}), it follows that
\begin{align*}
\prr\left\{ \loss(\algo(\Z);Z) \in \widehat{\EE}_{\alpha}(Z;\algo) \right\} \geq 1-\alpha.
\end{align*}
Furthermore, when the losses $\loss(\algo_{\CP}^{(j)};Z_{j})$ (for $j \in \II_{\CP}^{\prime}$) and $\loss(\algo(\Z);Z)$ are all almost surely distinct, we also have
\begin{align*}
\prr\left\{ \loss(\algo(\Z);Z) \in \widehat{\EE}_{\alpha}(Z;\algo) \right\} \leq 1-\alpha + \frac{1}{|\II_{\CP}^{\prime}|+1}.
\end{align*}
\end{thm}
\begin{rmk}
Note that the preceding theorem lets us make predictive statements which are qualitatively very different from those which are possible using the ``$Z$-free'' result in Theorem \ref{thm:cpp_algo_zfree}. Having constructed the CPP set $\widehat{\EE}_{\alpha}$ using the initial data $\Z_{n}$, the idea is that we can now evaluate $\widehat{\EE}_{\alpha}(z)$ at specific new points which are of domain-specific interest, making predictions about how we expect the algorithm $\algo$ to perform at these points, when passed an as-yet unobserved new training sample.\hfill$\blacksquare$
\end{rmk}

\subsubsection{Variable-width CPP sets}\label{sec:cpp_algo_zmod_variable}

Finally, we note that the CPP sets $\widehat{\EE}_{\alpha}(z)$ constructed in (\ref{eqn:cpp_algo_zmod_fixed_predset}) have a \emph{fixed} width. Passing different choices of $z$ certainly leads to different predictive intervals, but all of these intervals have the exact same width, equal to $2\qnt_{\alpha_{\CP}}[\S_{\CP}^{\prime}]$. As a natural means of constructing variable-width CPP sets, we adapt the conformalized quantile regression techniques of \citet{romano2020a} to our problem setting. The overall procedure is almost identical to our fixed-width CPP set construction, with two key differences. The first difference is that instead of using the single predictor $\widehat{f}_{\TR}(z)$, which is intended to approximate the central tendency of $\loss(\algo(\Z);Z)$ conditioned on $\{Z=z\}$, we use a pair of predictors $\widehat{q}_{\textup{lo}}(z)$ and $\widehat{q}_{\textup{hi}}(z)$, respectively intended to approximate the conditional quantile $\qnt_{\alpha}[\loss(\algo(\Z);Z) \cond Z=z]$ at low and high levels. The second difference is that since we now have two ancillary predictors instead of one, we adjust the non-conformity scores to reflect the worst of the two predictors.

Spelling this all out more explicitly, given the sample $\Z_{n}$, we break up the index $[n] = \II_{\TR} \cup \II_{\CP}$ into sub-indices precisely as was done in (\ref{eqn:cpp_algo_zmod_fixed_partition}). We then pass the same $(\text{input},\text{response})$ pairs as before to a sub-routine for doing \emph{quantile} regression,\footnote{There are numerous methods available for general-purpose quantile regression, from classical linear quantile regression \citep{koenker1978a}, to modern techniques using kernel methods \citep{takeuchi2006a}, random forests \citep{meinshausen2006a}, and neural networks \citep{taylor2000a}.} denoted by $\texttt{QuantReg}[\cdot]$, to obtain
\begin{align*}
(\widehat{q}_{\textup{lo}}, \widehat{q}_{\textup{hi}}) = \texttt{QuantReg}\left[ \left\{ (Z_{j},\loss(\algo_{\TR}^{(j)};Z_{j})): j \in \II_{\TR}^{\prime} \right\} \right].
\end{align*}
The simplest approach is to fix some small $0 < \alpha < 1$, and take $\widehat{q}_{\textup{lo}}$ and $\widehat{q}_{\textup{hi}}$ respectively as the outputs of $\texttt{QuantReg}$ for levels $\alpha/2$ and $1-\alpha/2$. With these predictors in hand, we can compute non-conformity scores as
\begin{align}\label{eqn:cpp_algo_zmod_variable_scores}
\S_{\CP}^{\prime\prime} \defeq \left\{ \max\left\{\loss(\algo_{\CP}^{(j)};Z_{j})-\widehat{q}_{\textup{hi}}(Z_{j}), \widehat{q}_{\textup{lo}}(Z_{j})-\loss(\algo_{\CP}^{(j)};Z_{j})\right\}: j \in \II_{\CP}^{\prime} \right\},
\end{align}
and then construct variable-width prediction intervals as
\begin{align}\label{eqn:cpp_algo_zmod_variable_predset}
\widehat{\EE}_{\alpha}(z;\algo) \defeq \left\{ u \in \RR: \widehat{q}_{\textup{lo}}(z)-\qnt_{\alpha_{\CP}}\left[\S_{\CP}^{\prime\prime}\right] \leq u \leq \widehat{q}_{\textup{hi}}(z)+\qnt_{\alpha_{\CP}}\left[\S_{\CP}^{\prime\prime}\right] \right\}
\end{align}
where once again $\alpha_{\CP} \defeq (1+1/|\II_{\CP}^{\prime}|)(1-\alpha)$. Since different predictors are being used for the two ends of the predictive interval, it is clear that the CPP set given by (\ref{eqn:cpp_algo_zmod_variable_predset}) has the flexibility to realize variable-width prediction intervals. The same basic quantile lemmas as we have used previously can be used to establish formal guarantees for this procedure as well.
\begin{thm}[Algorithm CPP: $Z$-modulated, variable-width case]\label{thm:cpp_algo_zmod_variable}
Given the CPP set $\widehat{\EE}_{\alpha}$ as defined in (\ref{eqn:cpp_algo_zmod_variable_predset}), let the partition of $\II_{\TR}$ and $\II_{\CP}$ be disjoint, with all subsets $\II_{\TR}^{(j)}$ and $\II_{\CP}^{(j)}$ having equal cardinality, say $m = |\II_{\TR}^{(j)}| = |\II_{\CP}^{(j)}|$ for all $j$. Let $\Z$ denote a fresh sample of size $m$. When the random variables $Z_{1},\ldots,Z_{n}$, $Z$, and $\Z$ are all exchangeable, and the algorithm $\algo$ is symmetric in the sense of (\ref{eqn:algo_symmetry}), we have
\begin{align*}
\prr\left\{ \loss(\algo(\Z);Z) \in \widehat{\EE}_{\alpha}(Z;\algo) \right\} \geq 1-\alpha.
\end{align*}
Furthermore, when the losses $\loss(\algo_{\CP}^{(j)};Z_{j})$ (for $j \in \II_{\CP}^{\prime}$) and $\loss(\algo(\Z);Z)$ are all almost surely distinct, we also have that
\begin{align*}
\prr\left\{ \loss(\algo(\Z);Z) \in \widehat{\EE}_{\alpha}(Z;\algo) \right\} \leq 1-\alpha + \frac{1}{|\II_{\CP}^{\prime}|+1}.
\end{align*}
\end{thm}

\subsection{Alternative performance definitions}

In the preceding sections, our notion of algorithm ``performance'' was constrained to a very specific, albeit typical, meaning: the value of an as-yet unobserved loss, which takes \emph{numerical} values, and is defined in a \emph{pointwise} fashion. Here we briefly discuss how the framework described above can easily be adapted to more diverse notions of performance.

\paragraph{Symbolic losses}

Consider for example the case of binary classification. While numerical losses are used for training, in the evaluation stage we are typically just concerned with performance \emph{categories}, whether a particular point was classified correctly, or if incorrectly, whether it was a false positive or negative. That is, we have categorical losses for evaluation taking values in a set of symbols, such as
\begin{align*}
\catloss(h;z) \in \{ \texttt{Correct}, \texttt{FP}, \texttt{FN} \}.
\end{align*}
Here we use $\catloss$ (instead of $\loss$) to distinguish between categorical and numerical losses. In this case, the CPP approach described in the preceding sections based on predictive intervals is meaningless, since the loss values returned by $\catloss(h;z)$ no longer have any numerical interpretation. That said, meaningful CPP sets can readily be constructed using a strategy which is analogous to the one taken in section \ref{sec:cpp_algo_zmod_variable}, albeit slightly more involved and data-intensive, using for example the conformalized quantile classification (CQC) approach of \citet{cauchois2020a}. Since the raw losses are discrete symbols, one must obtain real-valued scores somehow. Essentially, the idea is to learn appropriate scores via one additional sub-routine $\texttt{Scoring}[\cdot]$, at the cost of an extra top-level partition of the data. These learned scores are then passed through a workflow almost identical to the CPP set construction in section \ref{sec:cpp_algo_zmod_variable}. To briefly illustrate this, we split up the data now into three sets, say $[n] = \II_{\TR} \cup \II_{\EV} \cup \II_{\CP}$, and similar to section \ref{sec:cpp_algo_zmod_fixed}, we sub-partition further as
\begin{align*}
\II_{\TR} = \II_{\TR}^{\prime} \cup \left( \bigcup_{j \in \II_{\TR}^{\prime}} \II_{\TR}^{(j)} \right),%
\quad%
\II_{\EV} = \II_{\EV}^{\prime} \cup \left( \bigcup_{j \in \II_{\EV}^{\prime}} \II_{\EV}^{(j)} \right),%
\quad%
\II_{\CP} = \II_{\CP}^{\prime} \cup \left( \bigcup_{j \in \II_{\CP}^{\prime}} \II_{\CP}^{(j)} \right).
\end{align*}
Then using the sub-routine for scoring, one obtains a scoring function
\begin{align*}
\widehat{s} = \texttt{Scoring}\left[ \left\{ (Z_{j}, \catloss(\algo_{\TR}^{(j)};Z_{j})) : j \in \II_{\TR}^{\prime} \right\} \right].
\end{align*}
As a typical example of how to implement $\texttt{Scoring}[\cdot]$, one might take the $(\text{instance},\text{label})$ pairs that are passed to the sub-routine and use them to train a classifier; the values returned by $\widehat{s}(z,l)$ for any arbitrary $z \in \ZZ$ and $l \in \{\texttt{Correct}, \texttt{FP}, \texttt{FN} \}$ then might naturally be the value of a numerical loss (e.g., log-loss) incurred by this learned classifier at the point $(z,l)$. Now that we have an ``analog'' signal to work with, one then uses the second subset to obtain a conditional quantile estimator for these scores as
\begin{align*}
\widehat{q}_{\text{hi}} = \texttt{QuantReg}\left[ \left\{ \left( Z_{j},\widehat{s}(Z_{j},\catloss(\algo_{\EV}^{(j)};Z_{j})) \right): j \in \II_{\EV} \right\} \right],
\end{align*}
noting the key role played by $\widehat{s}$ here. Using the final remaining subset to construct non-conformity scores as
\begin{align*}
\S_{\CP}^{\prime\prime} = \left\{ \widehat{q}_{\text{hi}}(Z_{j})-\widehat{s}(Z_{j},\catloss(\algo_{\EV}^{(j)};Z_{j})): j \in \II_{\CP}^{\prime} \right\},
\end{align*}
one can naturally construct meaningful CPP sets $\widehat{\EE}_{\alpha}(z)$ as
\begin{align*}
\widehat{\EE}_{\alpha}(z) \defeq \left\{l \in \{\texttt{Correct}, \texttt{FP}, \texttt{FN}\}: \widehat{s}(z,l) \geq \widehat{q}_{\text{hi}}(z) - \qnt_{\alpha_{\CP}}\left[ \S_{\CP}^{\prime\prime} \right]  \right\},
\end{align*}
with $\alpha_{\CP}$ defined as in section \ref{sec:cpp_algo_zmod_variable}. Once again, a standard argument using the off-sample quantile lemmas applied to the non-conformity scores yields
\begin{align*}
1-\alpha \leq \prr\left\{ \catloss(\algo(\Z);Z) \in \widehat{\EE}_{\alpha}(Z;\algo) \right\} \leq 1-\alpha + \frac{1}{|\II_{\CP}^{\prime}|+1},
\end{align*}
for a new draw of $Z$ and sample $\Z$. For the above argument, we used $\{\texttt{Correct}, \texttt{FP}, \texttt{FN}\}$ as a concrete example of categorical loss values, but clearly this approach does not rely on any special properties of the symbols being used, and generalizes to arbitrary countable sets of symbols.

\paragraph{Sample-based losses}

Our focus up to this point has been on evaluating performance in a pointwise fashion, i.e., evaluating a candidate $h \in \HH$ based on a single point $z \in \ZZ$, quantified in the form of a loss $\loss(h;z)$. Of course, it is common practice to evaluate using not just one point, but a separate sample of points, often referred to as the ``test set.'' Here we simply remark that at the cost of more data, one can trivially extend the CPP framework introduced in sections \ref{sec:cpp_candidate}--\ref{sec:cpp_algo} to sample-driven losses. As a simple, concrete example to show how this could be done, consider candidate CPP in section \ref{sec:cpp_candidate}, where $\algo_{\TR}$ is the candidate obtained by running $\algo$ on the training set $\Z_{\TR}$. Now, instead of using the calibration data $\Z_{\CP}$ as-is, one would need to split it up into $k$ subsets for constructing sample-based losses, as
\begin{align*}
\II_{\CP} = \bigcup_{j=1}^{k} \II_{\CP}^{(j)},
\end{align*}
where we assume that each subset $\II_{\CP}^{(j)}$ has the same cardinality, say $m=|\II_{\CP}^{(j)}|$ for all $j \in [k]$. Then, if we have a loss function $\loss(h;\{z_{1},z_{2},\ldots\})$ that is defined for collections of one or more points, we can simply re-christen the set $\losses_{\CP}$ used in section \ref{sec:cpp_candidate} to be
\begin{align*}
\losses_{\CP} \defeq \left\{ \loss(\algo_{\TR};\{Z_{i}: i \in \II_{\CP}^{(j)}\}) : j \in [k] \right\},
\end{align*}
and the validity/calibration guarantees of Theorem \ref{thm:cpp_candidate} can be easily extended to yield
\begin{align*}
1-\alpha \leq \prr\left\{ \loss(\algo_{\TR};\{Z_{1}^{\prime},\ldots,Z_{m}^{\prime}\}) \in \widehat{\EE}_{\alpha}(\algo_{\TR}) \right\} \leq 1-\alpha + \frac{1}{k+1},
\end{align*}
where $Z_{1}^{\prime},\ldots,Z_{m}^{\prime}$ denotes a fresh testing sample, up to the usual caveats regarding exchangeability and tie-breaking. Setting $k=|\II_{\CP}|$ is the special case of pointwise losses, i.e., $m=1$. Taking a larger $k$ gives more data for conformal inference, but introduces more noise into the evaluation process due to sampling a smaller set. This strategy can be applied in an analogous fashion to extend all the constructions of CPP for algorithms in section \ref{sec:cpp_algo} to capture sample-based losses.

\section{Empirical examples}\label{sec:empirical}

In this section, we provide two natural examples of the proposed CPP framework being implemented and used to analyze different learning algorithms.

\subsection{Classification: comparing linear versus non-linear SVM}\label{sec:empirical_optdigits}

In our first example, we consider a multi-class classification problem using the well-known ``digits'' dataset.\footnote{\url{https://archive.ics.uci.edu/ml/datasets/Optical+Recognition+of+Handwritten+Digits}} We compare two types of support vector machine classifiers, one using the linear kernel (denoted \texttt{SVM-lin}), and the other using the radial basis function kernel (denoted \texttt{SVM-rbf}).\footnote{This is implemented using instances of the \texttt{svm.SVC} class in the scikit-learn library \citep{scikit-learn2011}, with default settings aside from the choice of kernel.} The classifiers are given raw images without any pre-processing.

\begin{figure}[t]
\centering
\includegraphics[width=0.49\textwidth]{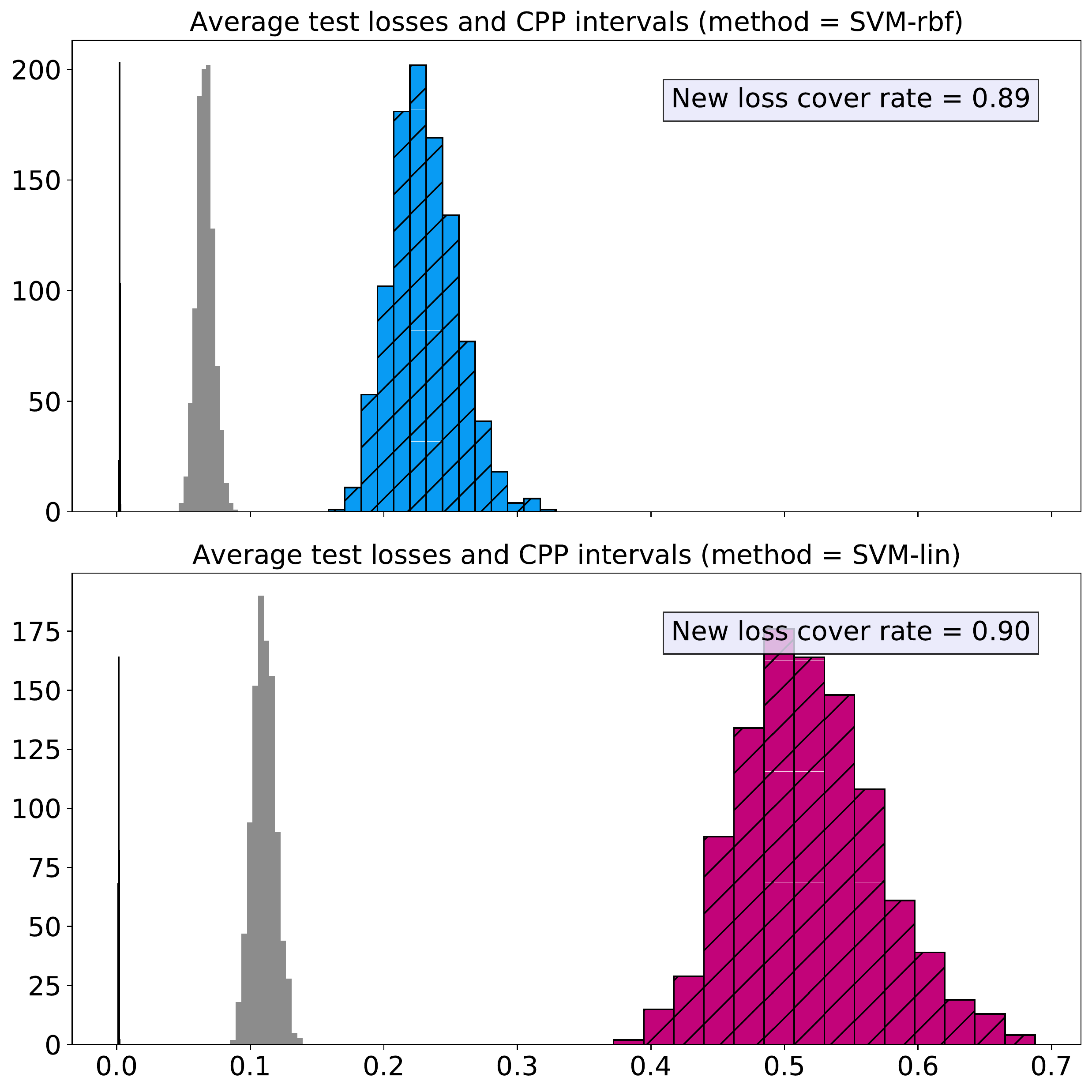}\,\includegraphics[width=0.49\textwidth]{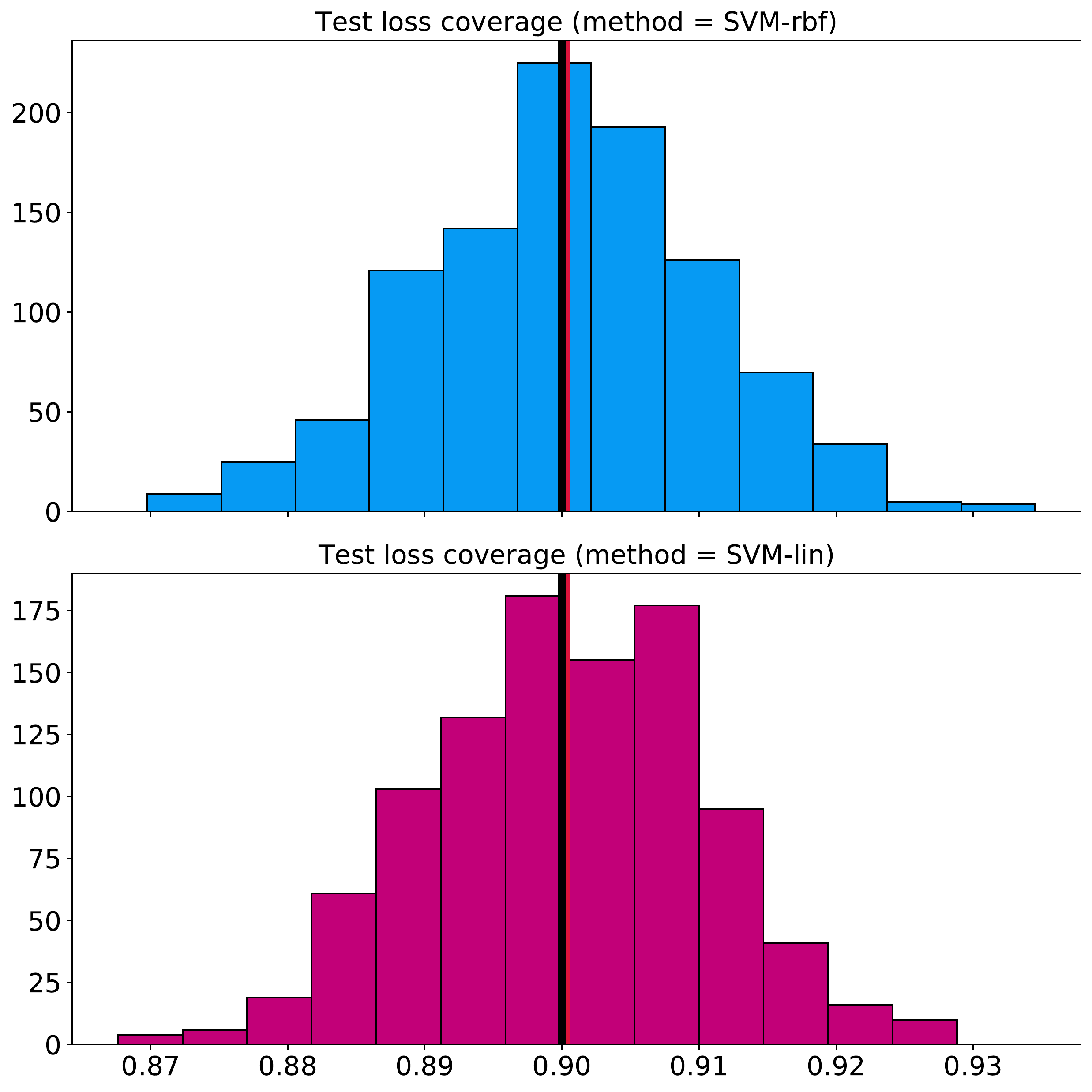}
\caption{Candidate CPP applied at $90\%$ confidence, for SVM-based multi-class classification, under the logistic loss.}
\label{fig:optdigits}
\end{figure}

As an initial example, we consider candidate CPP (section \ref{sec:cpp_candidate}), at a $90\%$ confidence level (i.e., $\alpha = 0.1$). We run $1000$ trials, and for each trial, we randomly choose $75\%$ of the full dataset to be used as $\Z_{n}$; the rest is used for testing. $\Z_{n}$ is then split into $\Z_{\TR}$ and $\Z_{\CP}$, with $|\II_{\TR}|=|\II_{\CP}|$. For $\loss(h;z)$, we use the logistic loss.\footnote{We simply re-write \texttt{metrics.log\_loss} in scikit-learn to not sum over data points.} Representative results are given in Figure \ref{fig:optdigits}; see also Figure \ref{fig:ex_optdigits}. For each trial and each method, using $\Z_{\TR}$, we get $\algo_{\TR}$, and then using $\Z_{\CP}$, we get $\widehat{\EE}_{\alpha}(\algo_{\TR})$. In the left-most plots of Figure \ref{fig:optdigits}, we have histograms of ``average test losses'' (grey) and ``CPP intervals'' for each method (colored). By average test losses, we refer to the empirical mean of $\loss(\algo_{\TR};Z_{i})$ evaluated on the test data (the remaining $25\%$ of the full dataset). By CPP intervals, we mean the left and right sides of the interval characterizing $\widehat{\EE}_{\alpha}(\algo_{\TR})$; the left sides all tend to be very close to zero, and thus are difficult to see. We have drawn diagonal lines through the histogram of the upper bounds of $\widehat{\EE}_{\alpha}(\algo_{\TR})$ to distinguish them from the lower bounds. Also, for each trial, we randomly choose a \emph{single} point from the test set, say $Z_{\text{test}}$, and check whether or not $\loss(\algo_{\TR};Z_{\text{test}}) \in \widehat{\EE}_{\alpha}(\algo_{\TR})$. The fraction of the trials in which this condition turned out to be true is what we call the ``new loss cover rate.'' On the right-hand side of Figure \ref{fig:optdigits}, we have a histogram of ``test loss coverage.'' For each trial, the test loss coverage is simply the fraction of the test losses that are contained in $\widehat{\EE}_{\alpha}(\algo_{\TR})$. Finally, the top plot in Figure \ref{fig:ex_optdigits} gives the widths computed by subtracting the upper end of $\widehat{\EE}_{\alpha}(\algo_{\TR})$ from the lower end, for each method and each trial.

Next, in order to illustrate one potentially interesting application of $Z$-modulated algorithm CPP, we run a single trial, and for concreteness we do variable-width intervals (section \ref{sec:cpp_algo_zmod_variable}) at $90\%$ confidence. Once again $75\%$ of the full dataset is used as $\Z_{n}$, with the rest used for testing, and again we have $|\II_{\TR}|=|\II_{\CP}|$. The sub-indices are created such that $|\II_{\TR}^{\prime}|=|\II_{\CP}^{\prime}|=50$. For \texttt{QuantReg}, we use random forest quantile regression \citep{meinshausen2006a}, but instead of passing the raw pairs $(Z_{j},\loss(\algo_{\TR}^{(j)};Z_{j}))$ for $j \in \II_{\TR}^{\prime}$ to this sub-routine, we normalize pixel values to the unit interval, and then reduce the dimension of the image part of $Z_{j}$ from $8 \times 8 = 64$ to $4$, using vertical and horizontal symmetry,\footnote{We define vertical/horizontal \emph{asymmetry} as the mean of the absolute pixel-wise differences computed between the original image and the image reflected along the horizontal/vertical axis. Symmetry is then defined simply as $\text{Symmetry }= (-1) \times \text{Asymmetry}$.} mean pixel value, and the standard deviation of pixel values. Finally, the label part of $Z_{j}$ is turned into a one-hot vector ($10$ classes here). Thus, the dimension of the modified ``inputs'' being fed to \texttt{QuantReg} have dimension $4 + 10 = 14$. Using this feature transformation included in the \texttt{QuantReg} sub-routine, for each random trial, we obtain $\widehat{\EE}_{\alpha}(z;\algo)$ for a single trial. We give two examples, one in the bottom plot of Figure \ref{fig:ex_optdigits}, and one in Figure \ref{fig:optdigits_zmod}. In both figures, we have chosen a random data point from the test data (shown on the left), say $Z_{\text{test}}$. Shown on the right, we have $\widehat{\EE}_{\alpha}(Z_{\text{test}};\algo)$ drawn as colored bars for each method, with black diamonds showing the actual loss $\loss(\algo(\Z);Z_{\text{test}})$ incurred given a fresh sample $\Z$ from the test data.

\begin{figure}[t]
\centering
\includegraphics[width=0.75\textwidth]{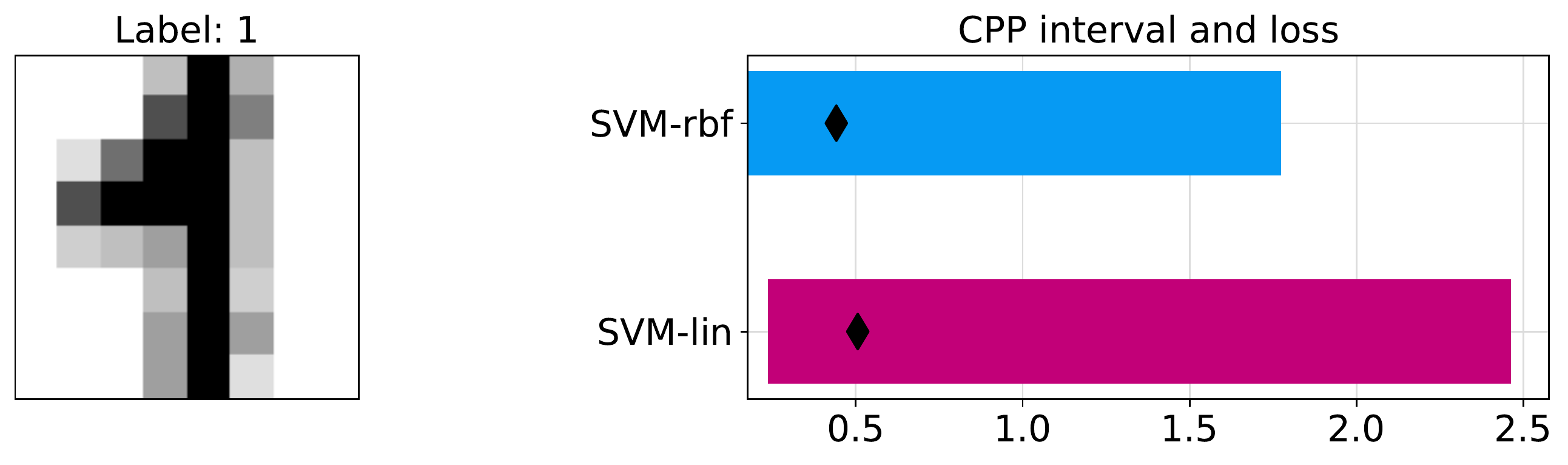}
\caption{Algorithm CPP ($Z$-modulated, variable width) at $90\%$ confidence, used for linear/non-linear SVM applied to the digits data.}
\label{fig:optdigits_zmod}
\end{figure}

\subsection{Stochastic convex optimization: comparing batch GD with SGD}\label{sec:empirical_quadnoise}

In our second example, we consider a simulated regression problem using a simple linear model with additive noise. We generate pairs $Z=(X,Y)$ with the relation $Y = \langle \wstar, X \rangle + \epsilon$, where $\wstar$ is a $d$-dimensional vector of $1$s, the inputs $X$ follow an isotropic multivariate Normal distribution (with unit variance along the diagonal), and noise $\epsilon$ follows a zero-mean Normal distribution with standard deviation of $2.2$ for the candidate CPP case, and a Student-t distribution with degrees of freedom set to $2.1$ for the algorithm CPP case to further highlight the impact of more heavy-tailed errors. We compare two iterative learning algorithms, namely empirical risk minimization implemented by batch gradient descent (step size $0.1$, denoted \texttt{GD\_ERM}) and stochastic gradient descent (step size $0.01$, denoted \texttt{SGD\_ERM}). Here the ``gradients'' we refer to are the gradients of the squared error taken with respect to the candidate weights at any given iteration. Both methods are randomly initialized by selecting each initial weight randomly from $\text{Uniform}[-5,5]$. We give both methods a fixed ``budget'' of $5|\II_{\TR}|$. That is, batch GD is only allowed five updates, whereas SGD is allowed to take five full passes.

Here as well, we consider both candidate and algorithm CPP at $90\%$ confidence, using the squared error for $\loss(h;z)$. First, implementing candidate CPP, we run $1000$ randomized trials, and generate a sample $\Z_{n}$ of independent pairs as described in the previous paragraph, with $n=7500$, and $|\II_{\TR}|=|\II_{\CP}|$. The test set size is $2500$. Representative results are given in Figure \ref{fig:quadnoise_candidate_normal}, where average test losses, CPP intervals, new loss cover rates, and test loss coverage are all computed in the exact same fashion as described in the previous sub-section for Figure \ref{fig:optdigits}. The widths of the CPP intervals in the top plot of Figure \ref{fig:quadnoise_candidate_normal} for both methods are plotted together in Figure \ref{fig:ex_quadnoise}. Next, we consider implementing $Z$-free algorithm CPP (section \ref{sec:cpp_algo_zfree}). We run again $1000$ randomized trials, where for each trial, we generate a sample $\Z_{n}$ with $n=75000$ and $|\II_{\EV}|=1000$, so $|\II_{\TR}^{(j)}|=(n-|\II_{\EV}|)/|\II_{\EV}|=74$ for each $j \in \II_{\EV}$. Representative results are given in Figure \ref{fig:quadnoise_algorithm_student}, with average test losses, CPP intervals, new loss cover rates, and test loss coverage all defined in an analogous fashion to the previous figure, now using the CPP interval $\widehat{\EE}_{\alpha}(\algo)$ from (\ref{eqn:cpp_algo_zfree_predset}) instead of (\ref{eqn:cpp_candidate_predset}).

\begin{figure}[t]
\centering
\includegraphics[width=0.49\textwidth]{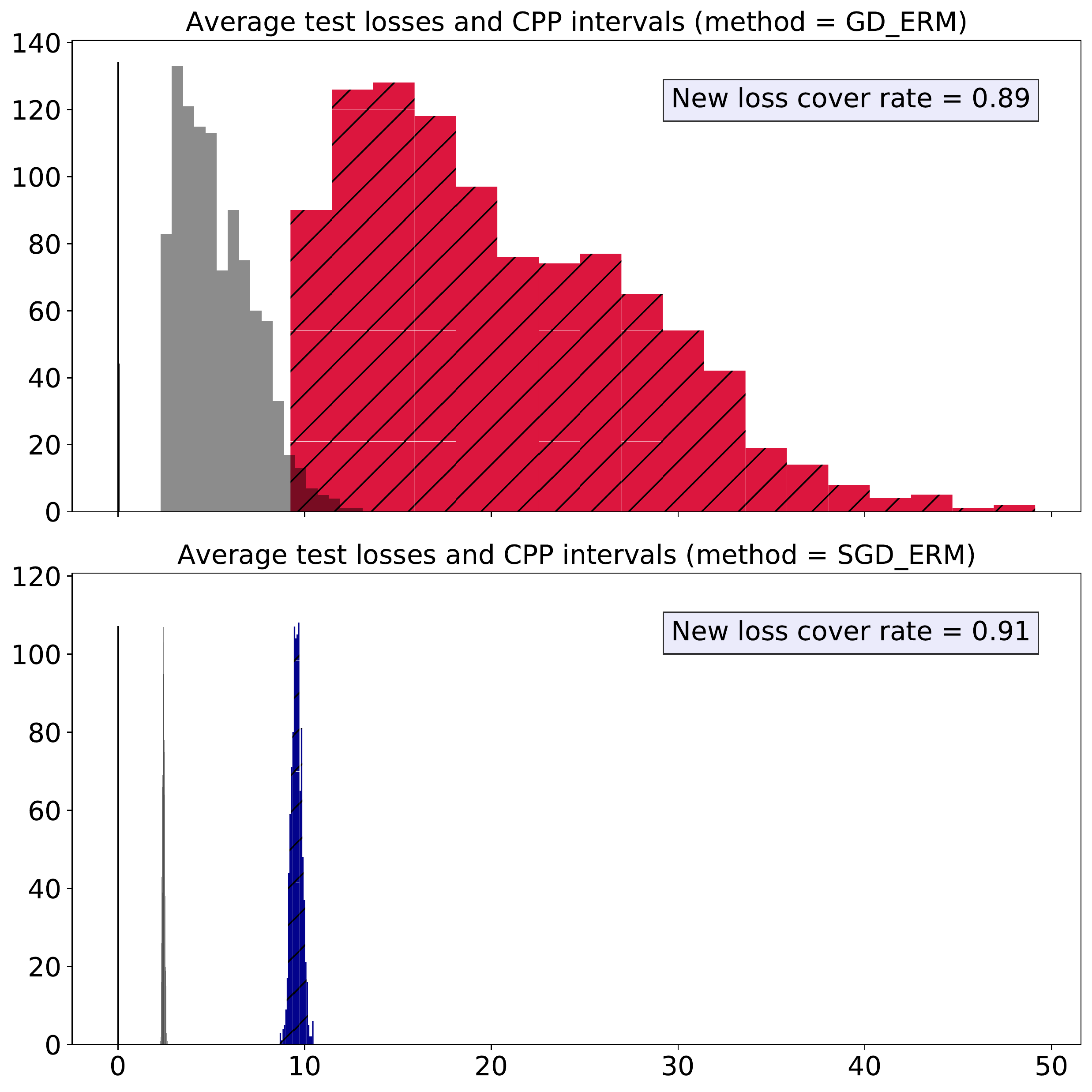}\,\includegraphics[width=0.49\textwidth]{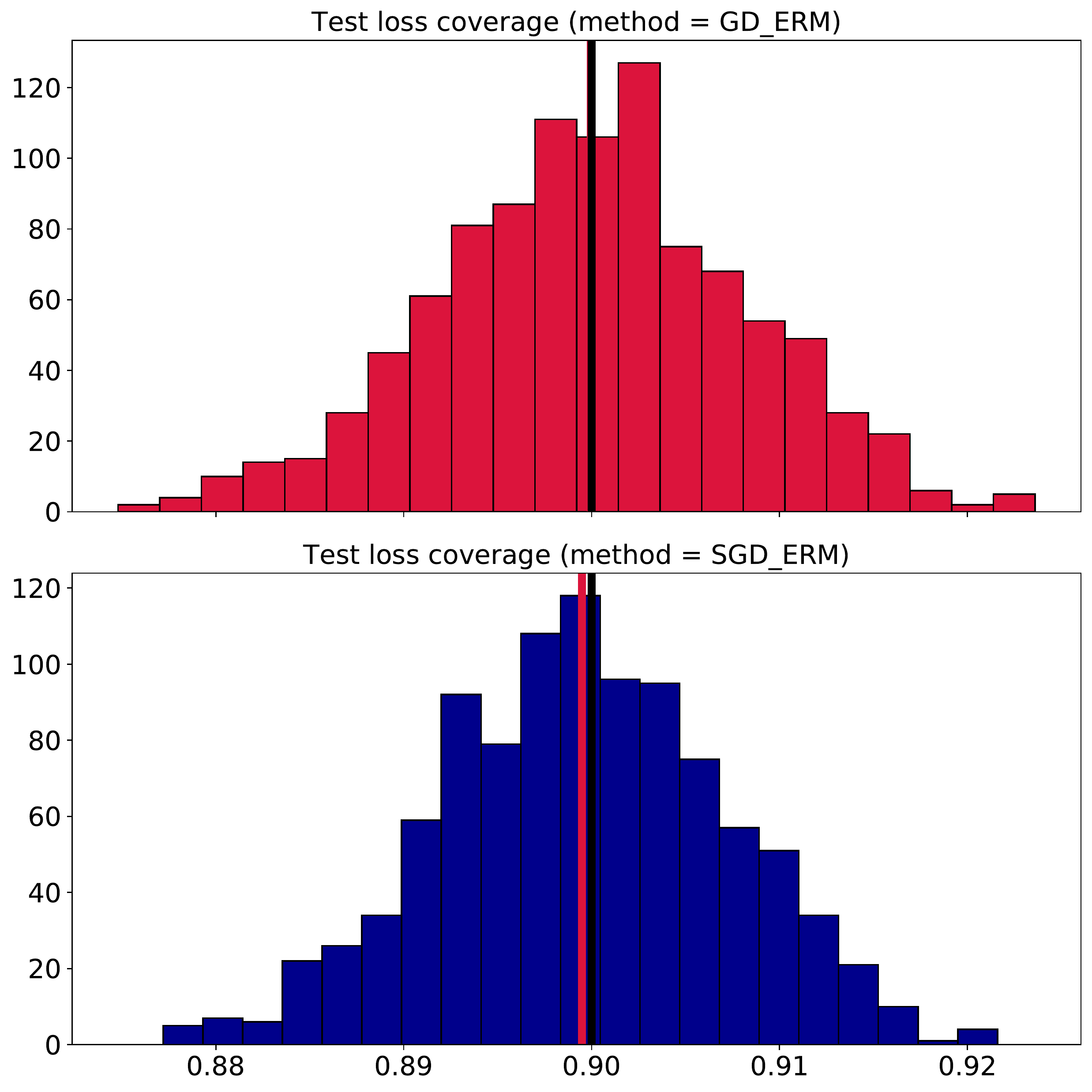}
\caption{Candidate CPP applied at $90\%$ confidence, for gradient-based stochastic convex optimization, under the squared error.}
\label{fig:quadnoise_candidate_normal}
\end{figure}

\begin{figure}[t]
\centering
\includegraphics[width=0.49\textwidth]{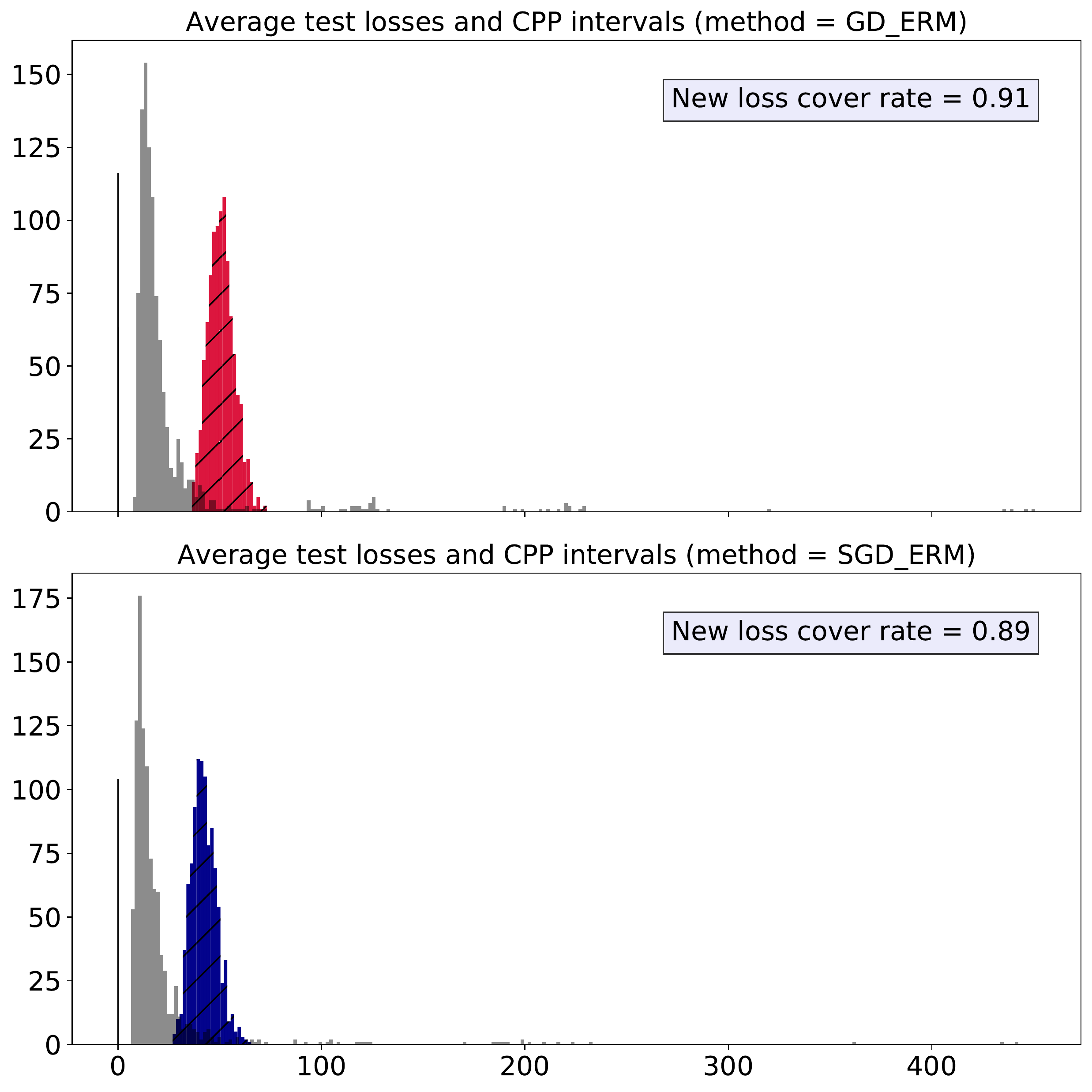}\,\includegraphics[width=0.49\textwidth]{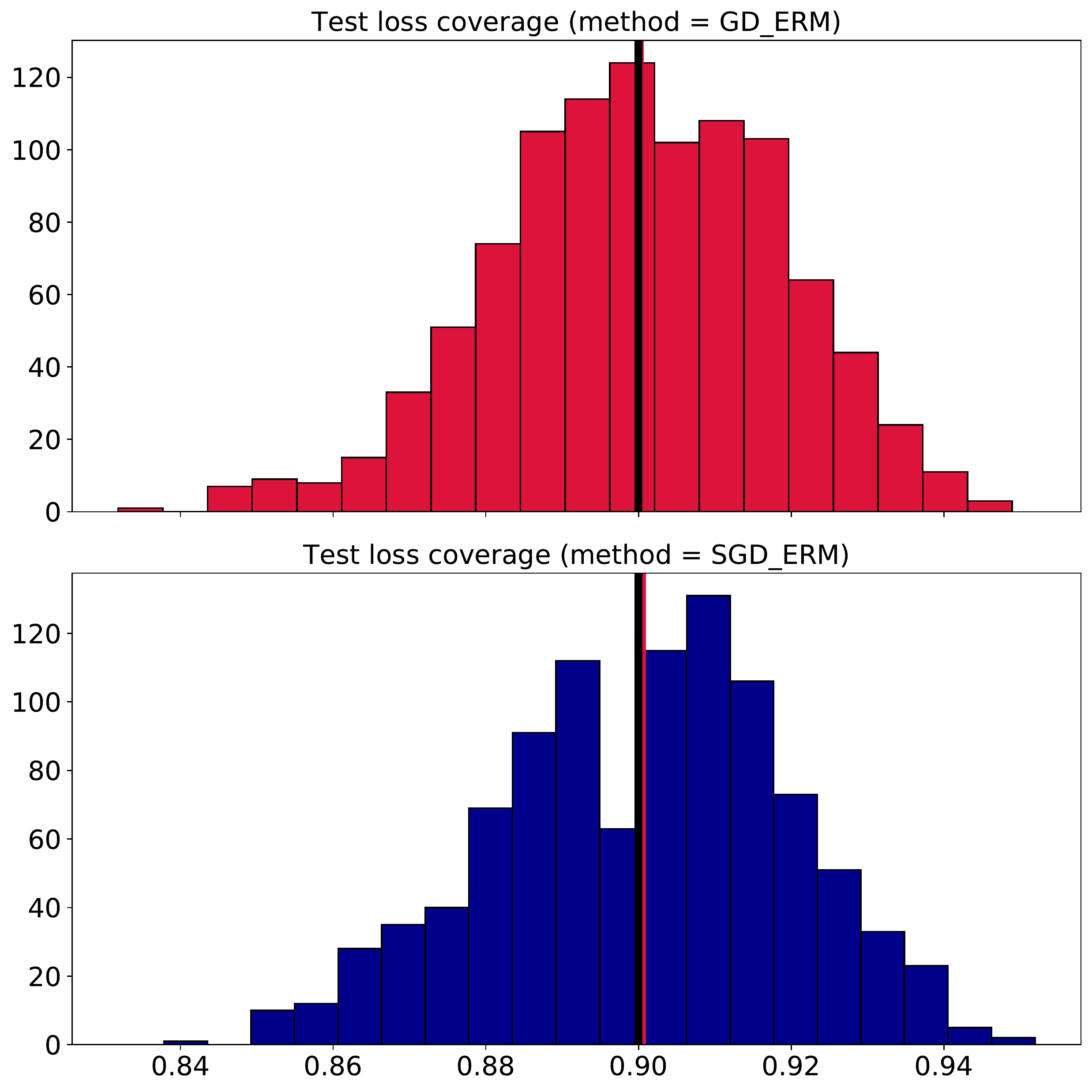}
\caption{Algorithm CPP ($Z$-free) at $90\%$ confidence, for gradient-based stochastic convex optimization, under the squared error.}
\label{fig:quadnoise_algorithm_student}
\end{figure}

\section{Future directions}

The scope of this first paper has been restricted to just the most basic formulation of the proposed CPP framework, plus the background needed to analyze this framework and a handful of empirical examples as an illustrative proof of concept. Due the generality of this methodology, there are numerous possibilities for future research directions, both in terms of basic research and more applied work. Compared with traditional ``candidate'' type conformal prediction, the algorithm-centric approach discussed here is much more data-intensive, and coverage guarantees can be easily confirmed to erode away when the conformal approach is replaced by a naive cross-validation procedure. Thus, more data-efficient approaches that can guarantee better off-sample coverage is a natural point of interest. Dealing with settings where exchangeability is not satisfied (e.g., a covariate shift scenario \citep{tibshirani2020a}) is certainly important and worth investigating, including potential links to detection of model shift. Going in a different direction, using computable indicators of off-sample performance as part of a larger algorithmic strategy, and the potential for links to more traditional risk-centric PAC-learning guarantees is also of interest. On the applied side, the black-box sub-routines $\texttt{Regression}$ and $\texttt{QuantReg}$ may be tasked with learning a complex and highly non-linear relation between the data and algorithm performance. As such, both model and algorithm decisions with respect to this sub-routine are in fact critical to the practical utility of all conformal prediction approaches, and domain-specific applied work aimed at establishing some best practices is assuredly of great value moving forward.

\appendix

\section{Additional proofs}\label{sec:apdx_proofs}

\begin{proof}[Proof of Theorem \ref{thm:cpp_candidate}]
To begin, we use generic random variables $\U_{n} = (U_{1},\ldots,U_{n})$, assumed to be exchangeable. First, observe that for any $i \in [n]$ and $0 < \beta < 1$, we have
\begin{align}\label{eqn:cpp_candidate_0}
\prr\left\{ U_{i} < \qnt_{\beta}\left[ \U_{n} \right] \right\} = \prr\left\{ U_{i} < U_{(\lceil n\beta \rceil,n)} \right\} \leq \prr\left\{ U_{i} < U_{(\lfloor n\beta \rfloor+1,n)} \right\} = \prr\left\{ U_{i} < \qnt_{\beta}^{-}\left[ \U_{n} \right] \right\}.
\end{align}
The first and last equalities follow respectively from (\ref{eqn:quantile_order_relation}) and (\ref{eqn:rightquantile_order_relation}). The inequality follows by monotonicity and the fact that $\lceil n\beta \rceil \leq \lfloor n\beta \rfloor + 1$. Using (\ref{eqn:cpp_candidate_0}) along with (\ref{eqn:quantile_sum_upperbd}), and setting $\beta = \alpha/2$, we obtain
\begin{align}\label{eqn:cpp_candidate_1}
\prr\left\{ U_{i} < \qnt_{\alpha/2}\left[ \U_{n} \right] \right\} \leq \frac{\alpha}{2}.
\end{align}
This upper bound will be useful for lower-bounding the following two-sided event. Denoting the distribution functions of $U_{i}$ by $F_{i}$ and $F_{i}^{-}$, we can obtain a lower bound for on-sample two-sided intervals as
\begin{align}
\nonumber
\prr\left\{ \qnt_{\alpha/2}\left[\U_{n}\right] \leq U_{i} \leq \qnt_{1-\alpha/2}\left[\U_{n}\right] \right\} & = F_{i}\left(\qnt_{1-\alpha/2}\left[\U_{n}\right]\right) - F_{i}^{-}\left(\qnt_{\alpha/2}\left[\U_{n}\right]\right)\\
\nonumber
& = \prr\left\{ U_{i} \leq \qnt_{1-\alpha/2}\left[\U_{n}\right] \right\} - \prr\left\{ U_{i} < \qnt_{\alpha/2}\left[\U_{n}\right] \right\}\\
\nonumber
& \geq \left(1-\frac{\alpha}{2}\right) - \frac{\alpha}{2}\\
\label{eqn:cpp_candidate_2}
& = 1-\alpha. 
\end{align}
The inequality leading to (\ref{eqn:cpp_candidate_2}) follows from applying Lemma \ref{lem:quantile_onsample} to lower-bound the term being added, and (\ref{eqn:cpp_candidate_1}) to upper-bound the term being subtracted. Let $U$ be a newly drawn random variable such that $\U_{n}$ and $U$ are all exchangeable and almost surely distinct. From the exposition leading up to Lemma \ref{lem:quantile_onsample}, we already know that if we set $\beta_{\textup{hi}} \defeq (1+1/n)(1-\alpha/2)$, then we have the equivalence
\begin{align}\label{eqn:cpp_candidate_upper_equiv}
U \leq \qnt_{\beta_{\textup{hi}}}\left[ \U_{n} \right] \iff U \leq \qnt_{1-\alpha/2}\left[ \U_{n} \cup \{U\} \right].
\end{align}
As for the lower end of the desired interval, first recall that via (\ref{eqn:quantile_order_relation}), for any choice of $0 < \beta < 1$, we have
\begin{align}\label{eqn:cpp_candidate_3}
\qnt_{\beta}\left[ \U_{n} \right] = U_{(\lceil n\beta \rceil,n)} = U_{[n-\lceil n\beta \rceil+1,n]}.
\end{align}
We want to link this up with $\qnt_{\alpha/2}[\U_{n} \cup \{U\}]$ via an appropriately ``deflated'' quantile level $\beta < \alpha/2$. Just like in the case of upper bounds, the key observation is that comparing sorted elements of $\U_{n}$ before and after adding $\{U\}$, if $U$ is among the $k$th-largest elements of $\U_{n}$, it is also among the $k$th-largest elements of $\U_{n} \cup \{U\}$ (again, the converse clearly holds, see Figure \ref{fig:schematic_sorting}). Algebraically, we have that for any $1 \leq k \leq n$,
\begin{align}
U \geq U_{[k,n]} \iff U \geq U_{[k,n+1]}.
\end{align}
Using this basic observation and (\ref{eqn:quantile_order_relation}), we have
\begin{align}
\nonumber
U \geq \qnt_{\alpha/2}\left[ \U_{n} \cup \{U\} \right] & \iff U \geq U_{[n+2-\lceil (n+1)\alpha/2 \rceil,n+1]}\\
\label{eqn:cpp_candidate_4}
& \iff U \geq U_{[n+2-\lceil (n+1)\alpha/2 \rceil,n]}.
\end{align}
Comparing the right-most side of (\ref{eqn:cpp_candidate_3}) and (\ref{eqn:cpp_candidate_4}), we want to choose $\beta$ such that
\begin{align*}
n - \lceil n\beta \rceil + 1 = n + 2 - \lceil (n+1)\alpha/2 \rceil.
\end{align*}
It is easily checked that setting $\beta$ to $\beta_{\textup{lo}} \defeq \alpha/2 - (1-\alpha/2)/n$ achieves this. That is, we have
\begin{align}\label{eqn:cpp_candidate_lower_equiv}
U \geq \qnt_{\beta_{\textup{lo}}}\left[ \U_{n} \right] \iff U \geq \qnt_{\alpha/2}\left[ \U_{n} \cup \{U\} \right].
\end{align}
Taking (\ref{eqn:cpp_candidate_upper_equiv}) and (\ref{eqn:cpp_candidate_lower_equiv}) together, we have
\begin{align}\label{eqn:cpp_candidate_5}
\qnt_{\beta_{\textup{lo}}}\left[\U_{n}\right] \leq U \leq \qnt_{\beta_{\textup{hi}}}\left[\U_{n}\right] \iff \qnt_{\alpha/2}\left[\U_{n} \cup \{U\} \right] \leq U \leq \qnt_{1-\alpha/2}\left[ \U_{n} \cup \{U\} \right].
\end{align}
The left-hand side of (\ref{eqn:cpp_candidate_5}) is the event we are interested in. The right-hand side can be controlled by applying the basic fact (\ref{eqn:cpp_candidate_2}), but this time to the $(n+1)$-sized data set $\U_{n} \cup \{U\}$ including $U$, instead of $\U_{n}$. It immediately follows that
\begin{align}\label{eqn:cpp_candidate_6}
\prr\left\{ \qnt_{\beta_{\textup{lo}}}\left[\U_{n}\right] \leq U \leq \qnt_{\beta_{\textup{hi}}}\left[\U_{n}\right] \right\} \geq 1-\alpha.
\end{align}
It just remains to deal with the calibration part of the theorem statement. Assume that the generic random variables $\U_{n}$ and $U$ are almost surely distinct, in addition to being exchangeable. Then for each $i \in [n]$, using exchangeability, probabilities can be computed exactly, and we have
\begin{align}
\nonumber
\prr\left\{ \qnt_{\alpha/2}\left[\U_{n}\right] \leq U_{i} \leq \qnt_{1-\alpha/2}\left[\U_{n}\right] \right\} & = \frac{\lceil n(1-\alpha/2) \rceil - (\lceil n\alpha/2 \rceil-1)}{n}\\
\nonumber
& \leq \left(1-\frac{\alpha}{2}\right) + \frac{2}{n} - \frac{\lceil n\alpha/2 \rceil}{n}\\
\label{eqn:cpp_candidate_7}
& \leq 1-\alpha + \frac{2}{n}.
\end{align}
Note that we subtract $\lceil n\alpha/2 \rceil-1$ instead of $\lceil n\alpha/2 \rceil$, since the $\lceil n\alpha/2 \rceil$th-smallest element is included in the interval specified by the desired event. The inequalities follow immediately from the elementary inequality $x \leq \lceil x \rceil \leq x+1$. With this new upper bound (\ref{eqn:cpp_candidate_7}) in hand, another application of (\ref{eqn:cpp_candidate_5}) and (\ref{eqn:cpp_candidate_2}) yields
\begin{align}\label{eqn:cpp_candidate_8}
\prr\left\{ \qnt_{\beta_{\textup{lo}}}\left[\U_{n}\right] \leq U \leq \qnt_{\beta_{\textup{hi}}}\left[\U_{n}\right] \right\} \leq 1-\alpha + \frac{2}{n}.
\end{align}
To conclude the proof of validity and calibration, we respectively apply (\ref{eqn:cpp_candidate_6}) and (\ref{eqn:cpp_candidate_8}), using the correspondence $U_{i} \leftrightarrow \loss(\algo_{\TR};Z_{i})$ and $U \leftrightarrow \loss(\algo_{\TR};Z)$, noting that the sample size changes from $n$ to $|\II_{\CP}|$ due to splitting. This sample size is reflected in the settings of $\alpha_{\textup{lo}}$ and $\alpha_{\textup{hi}}$ used in the construction of (\ref{eqn:cpp_candidate_predset}).
\end{proof}

\begin{proof}[Proof of Theorem \ref{thm:cpp_algo_zfree}]
Recalling our previous proof of Theorem \ref{thm:cpp_candidate}, a perfectly analogous argument can be applied here, except with the key correspondence being $U_{i} \leftrightarrow \loss(\algo_{\TR}^{(i)};Z_{i})$ and $U \leftrightarrow \loss(\algo(\Z);Z)$. By requiring the partition of $\II_{\TR}$ to be a disjoint partition, we have $\II_{\TR}^{(i)} \cap \II_{\TR}^{(j)} = \emptyset$ for all $i \neq j$. Since we are ensuring all the training subsets $\Z_{\TR}^{(j)}$ and $\Z$ have the same number of elements, exchangeability of the data coupled with symmetry of $\algo$ immediately implies that the losses $\loss(\algo_{\TR}^{(i)};Z_{i})$ and $\loss(\algo(\Z);Z)$ are all exchangeable. Thus, we can apply (\ref{eqn:cpp_candidate_6}) and (\ref{eqn:cpp_candidate_8}) just as in the proof of Theorem \ref{thm:cpp_candidate}, using the correspondence just stated, to obtain the desired result.
\end{proof}

\begin{proof}[Proof of Theorem \ref{thm:cpp_algo_zmod_fixed}]
As mentioned in the main text, this result effectively follows from Theorem \ref{thm:cp_regression_split}, applied to a rather different regression problem than is typical in the literature. The correspondence between classical split conformal regression (Theorem \ref{thm:cp_regression_split}) and our current setting is as follows: for the regression data and sample sizes, we have
\begin{align*}
X \leftrightarrow Z, \quad Y \leftrightarrow \loss(\algo(\Z);Z), \quad X_{j} \leftrightarrow Z_{j}, \quad Y_{j} \leftrightarrow \loss(\algo_{\TR}^{(j)};Z_{j}), \quad n \leftrightarrow |\II_{\CP}^{\prime}|,
\end{align*}
and for the non-conformity scores, we have
\begin{align*}
S(X_{j},Y_{j}) \leftrightarrow |\widehat{f}_{\CP}(Z_{j})-\loss(\algo_{\CP}^{(j)};Z_{j})|, \quad S(X,Y) \leftrightarrow |\widehat{f}_{\TR}(Z)-\loss(\algo(\Z);Z)|.
\end{align*}
With this correspondence in place, the first statement (validity guarantee) follows from an application of the first part of Theorem \ref{thm:cp_regression_split}, since the exchangeability of the non-conformity scores here follows from the exchangeability of the pairs $(Z_{j},\loss(\algo_{\CP}^{(j)};Z_{j}))$ and $(Z,\loss(\algo(\Z);Z))$, which in turn follows from the exchangeability of the data and the symmetry of $\algo$, just as was shown in the proof of Theorem \ref{thm:cpp_algo_zfree}. For the second statement (calibration guarantee), note that if the losses are almost surely distinct, then so are the non-conformity scores, meaning we can apply the second part of Theorem \ref{thm:cp_regression_split} to yield the desired result.
\end{proof}

\begin{proof}[Proof of Theorem \ref{thm:cpp_algo_zmod_variable}]
First note that by the construction of $\widehat{\EE}_{\alpha}$ in (\ref{eqn:cpp_algo_zmod_variable_predset}), drawing a new point $Z$ and sample $\Z$, it follows immediately that $\loss(\algo(\Z);Z) \in \widehat{\EE}_{\alpha}(Z)$ is equivalent to the following:
\begin{align*}
\widehat{q}_{\textup{lo}}(Z)-\loss(\algo(\Z);Z) \leq \qnt_{\alpha_{\CP}}\left[\S_{\CP}^{\prime\prime}\right] \text{ and } \loss(\algo(\Z);Z)-\widehat{q}_{\textup{hi}}(Z) \leq \qnt_{\alpha_{\CP}}\left[\S_{\CP}^{\prime\prime}\right].
\end{align*}
In turn, this can be equivalently re-stated as
\begin{align*}
\max\left\{\loss(\algo(\Z);Z)-\widehat{q}_{\textup{hi}}(Z), \widehat{q}_{\textup{lo}}(Z)-\loss(\algo(\Z);Z)\right\} \leq \qnt_{\alpha_{\CP}}\left[\S_{\CP}^{\prime\prime}\right].
\end{align*}
Up to exchangeability, we can apply the off-sample quantile property of Lemma \ref{lem:quantile_offsample} to obtain the desired validity and calibration results, using the correspondence
\begin{align*}
U_{j} & \leftrightarrow \max\left\{\loss(\algo_{\CP}^{(j)};Z_{j})-\widehat{q}_{\textup{hi}}(Z_{j}), \widehat{q}_{\textup{lo}}(Z_{j})-\loss(\algo_{\CP}^{(j)};Z_{j})\right\}\\
U & \leftrightarrow \max\left\{\loss(\algo(\Z);Z)-\widehat{q}_{\textup{hi}}(Z), \widehat{q}_{\textup{lo}}(Z)-\loss(\algo(\Z);Z)\right\},
\end{align*}
and noting that the conformal prediction set size is $|\II_{\CP}^{\prime}|$. As for the exchangeability of these $|\II_{\CP}^{\prime}|+1$ random variables, just as in proof of Theorem \ref{thm:cpp_algo_zmod_fixed}, the losses are exchangeable due to the exchangeability of the data and the symmetry of $\algo$, which immediately implies that the non-conformity scores used here are also exchangeable.
\end{proof}

\bibliographystyle{apalike}
\bibliography{refs}

\end{document}